\PassOptionsToPackage{table,dvipsnames}{xcolor}
\documentclass[accepted]{uai2026} %
\pdfoutput=1

\usepackage[american]{babel}
\newif\iflong
\longfalse

\usepackage{natbib} %
\usepackage{mathtools} %
\usepackage{booktabs} %
\usepackage{tikz} %

\usepackage[mathscr]{euscript}
\usepackage{microtype}
\usepackage{graphicx}
\usepackage{amsmath}
\usepackage{amssymb}
\usepackage{adjustbox}
\usepackage{amsthm,thmtools}
\usepackage[notrig]{physics}
\usepackage{mycommands}
\usepackage{pascal-def}
\usepackage{thm-restate}
\usepackage{hyperref}
\usepackage[capitalise,noabbrev]{cleveref}
\usepackage[lined,boxed, ruled, vlined]{algorithm2e}
\usepackage{subcaption}
\usepackage{wrapfig}
\usepackage{multirow}
\usepackage{multicol}
\usepackage{url}
\allowdisplaybreaks[4]
\newtheorem{theorem}{Theorem}
\newtheorem{lemma}[theorem]{Lemma}
\newtheorem{corollary}[theorem]{Corollary}

\newtheorem*{proposition*}{Proposition}

\newtheorem{definition}[theorem]{Definition}

\newcommand{\ourparagraph}[1]{\textbf{#1\ }}

\newcommand\mytitle{Bound to Disagree: Generalization Bounds via Certifiable Surrogates}
\title{\mytitle}

\author[1]{\href{mailto:<mabaz21@ulaval.ca>?Subject=Your UAI 2026 paper}{Mathieu Bazinet}{}}
\author[1,2]{Valentina Zantedeschi}
\author[1]{Pascal Germain}
\affil[1]{%
Département d'informatique et génie logiciel\\
Universit\'e Laval\\
Québec, Qc, Canada
}
\affil[2]{%
    ServiceNow Research\\
    Montreal, Qc, Canada
}
  
\begin{document}
\maketitle

\begin{abstract}
Generalization bounds for deep learning models are typically vacuous,  not computable or restricted to specific model classes. In this paper, we tackle these issues by providing new disagreement-based certificates for the gap between the true risk of any two predictors. 
We then bound the true risk of the predictor of interest via a surrogate model that enjoys tight generalization guarantees, and by evaluating our disagreement bound on an unlabeled dataset.
We empirically demonstrate the tightness of the obtained certificates and showcase the versatility of the approach by 
training surrogate models leveraging three different frameworks: sample compression, model compression and PAC-Bayes theory. 
Importantly, such guarantees are achieved without modifying the target model, nor adapting the training procedure to the generalization framework.
\end{abstract}
    
\section{Introduction}

Deep neural networks have been the stars of the machine learning field for the last decade. Their empirical performance challenges the once-common wisdom that over-parameterized models should overfit the training data
\citep{zhang2016understanding}. 
This gap between practice and theory calls for refined theoretical frameworks for studying generalization.
The community has mainly turned to statistical learning theory to understand this phenomenon, producing generalization bounds based on the VC dimension \citep{vcdim}, the Rademacher and Gaussian complexities \citep[e.g.,][]{bartlett2002rademacher, pinto2025on}, information theory \citep[e.g.,][]{hellstrom2025generalization}, PAC-Bayes theory \citep{mcallester1998some}, sample compression theory \citep{littlestone1986relating} and model compression theory \citep[e.g.,][]{zhou2018nonvacuous}. 
Despite this breadth of approaches, progress has remained limited: most bounds are vacuous when applied to medium-to-large neural networks or apply only to a modified version of the model rather than the original predictor.

\begin{figure}[!t]
    \centering
    \includegraphics[width=\columnwidth]{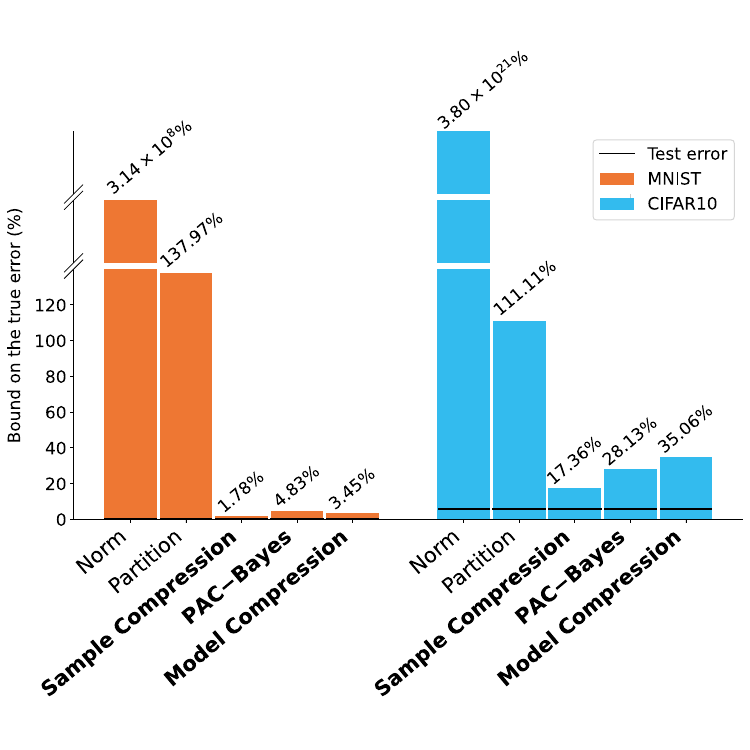}
    \caption{\textbf{Generalization bounds on the true risk of the model.} Comparison between bounds from the literature (Norm-based and Partition-based) and our new disagreement-based bounds, using surrogates from sample compression, PAC-Bayes and model compression theory.}
    \label{fig:summary}
\end{figure}

A recent line of work sidesteps these issues by leveraging a simpler surrogate model with similar behavior to the network of interest. 
The strategy is then to bound the true risk of the original model through its link with the surrogate.
In particular, \citet{hsu2021generalization} and \citet{Suzuki2020Compression} leverage a smaller compressed neural network as a surrogate model, and derive generalization bounds based either on distillation or on the Minkowski difference between the two models.
However, both results depend on universal constants that cannot be computed. 
In contrast, \citet{dziugaite2025size} provide PAC-Bayes generalization bounds by pruning neural networks and finding a ``teacher'' model of smaller size with small risk within them. 
While the resulting bounds are computable, this approach is restricted to neural networks with gated activations and residual connections, limiting its applicability. 
Finally, \citet{leblanc2025framework} recently presented a new approach for stochastic surrogates via a conditional decomposition of the loss. Unlike our approach, their method doesn't require a supplementary subset of the data, but cannot be used with stochastic neural networks for classification problems with more than two classes.

\newcommand{\gcheck}{\textcolor{OliveGreen}{$\pmb\checkmark$}}
\newcommand{\redx}{\textcolor{red}{$\pmb\times$}}
\newcommand{\noneed}{\gcheck} %
\newcommand{\need}{\redx} %

\begin{table}[!t]
    \centering
        \caption{Certification frameworks comparison. We seek deep neural network bounds that are computable, non-vacuous, and do not require assumptions about the target model. The last column shows whether an additional dataset beyond the train set is required (and should it be labeled or unlabeled). 
        } 
    \begin{adjustbox}{width=\columnwidth}
    \setlength{\tabcolsep}{3pt}
    \begin{tabular}{l@{\hspace{0pt}}ccccc}\toprule
    
     & Computable & Non-vacuous &  No assumptions & No additional data \\ \midrule
        PAC-Bayes & \gcheck & \gcheck&  \redx  & \noneed\\
        Sample Compression & \gcheck& \gcheck & \redx  & \noneed\\
        Model Compression & \gcheck & \gcheck &  \redx & \noneed\\
        VC dimension & \redx & \redx &  \gcheck & \noneed \\
        Information-theoretic & \gcheck & \gcheck &  \redx& \need{} (need labeled data)\\
        Norm-based &\gcheck & \redx &   \gcheck & \noneed\\
        Partition-based & \gcheck & \redx & \gcheck & \noneed \\
        \textbf{Our bound} (Zero-one loss) &\gcheck & \gcheck  & \gcheck & \need{}(need unlabeled data)  \\
        \textbf{Our bound} (Lipschitz loss) &\gcheck & \gcheck  & \gcheck & \need{} (need unlabeled data)  \\
        \textbf{Our bound} (Non-Lipschitz loss)  &\gcheck & \gcheck  & \gcheck & \need{} (need labeled data)  \\\bottomrule
    \end{tabular}
    \end{adjustbox}
    \label{tab:certification_frameworks}
\end{table}

In this paper, we present a framework that fulfills all desiderata simultaneously: 
we provide fully computable, non-vacuous bounds that hold with high probability and apply to any predictor, with no restrictions on architecture, activation functions, training objectives, or optimizer (see \cref{tab:certification_frameworks} for an overview of desiderata fulfilled by each framework). 
Like prior work, we leverage a surrogate model that is simple enough to enjoy tight generalization guarantees on its own, while remaining close enough to the target network that their predictions align on most inputs. The performance gap between the two models is measured through a disagreement bound,
which we evaluate on a small unlabeled dataset not used during training, the only additional requirement of our approach. In comparison to labeled data, acquiring unlabeled data is much cheaper and faster to obtain \citep{liao2021towards}, making this requirement much less stringent than labeled data. 

Our key theoretical contributions are reported in Section~\ref{sec:main}, where we present a sketch of our disagreement bound, before specializing this result to the zero-one loss and to Lipschitz-continuous losses. We further provide generalization certificates for the case where the disagreement term is optimized on the available unlabeled data.

The resulting certificates are much stronger than the other bounds in the literature that are computable and that hold for the target model without modification, as illustrated in \cref{fig:summary}.
Moreover, unlike previous approaches, our guarantees can be derived for virtually any proxy model.
We showcase this versatility in Section~\ref{sec:experiments}, where we apply three different approaches (sample compression, model compression, and PAC-Bayes theory) to train and certify the proxy model, and study various families of neural networks: a CNN on MNIST \citep{lecun1998_mnist}, a ResNet18 \citep{he2016deep} on CIFAR10 \citep{cifar10}, and a DistilBERT \citep{sanh2019distilbert} and a GPT2 \citep{radford2019language} on Amazon polarity \citep{zhang2015character}.   

    \section{Background and Notation}\label{sec:background}

Let $(\bx_1, y_1), \ldots, (\bx_n, y_n)$ be a sequence of $n$ independently and identically distributed (\emph{i.i.d.}) data points sampled from an unknown distribution $\calD$ over $\calX \times \calY$. Let $S = \{(\bx_i, y_i)\}_{i=1}^n$ be the dataset that contains these data points. In this paper, we consider a feature space $\calX \subseteq \R^d$ and a label space $\calY = \{1, 2, \ldots, C\}$ for $C$-class classification tasks.

The hypothesis class $\calH$ contains predictors $h : \calX \to \R^{C}$ with $h(\bx) = (h(\bx)_1, \ldots, h(\bx)_C)$. A learning algorithm $A : \bigcup_{k=1}^{\infty}(\calX \times \calY)^k \to \calH$ outputs a predictor $A(S) \in \calH$ when applied to $S$. Let $\ell : \R^C \times \calY \to [B_{\ell}, T_{\ell}]$ denote a loss function with a range $\lambda_{\ell} = T_{\ell} - B_{\ell}$. 
Given a loss function~$\ell$, the true loss of a predictor $h$ is
\begin{equation*}
    \calL_{\calD}(h) = \E_{(\bx,y) \sim \calD} \ell(h(\bx), y).
\end{equation*}
The true loss cannot be computed, as the distribution $\calD$ is unknown. However, given a dataset $S {\,\sim} \calD^n$, we can compute the empirical loss of a predictor \mbox{$\hatL_{S}(h) = \frac{1}{n} \sum_{i=1}^n \ell(h(\bx_i), y_i).$} When the choice of $h \in \calH$ is dependent on the dataset $S$, the empirical risk becomes a biased estimator of the true risk. Thus, we rely on generalization bounds to upper-bound the true risk.

\begin{definition}
A generalization bound is defined as a high-probability upper-bound on the true risk of a hypothesis $h \in \calH$. With a confidence parameter $\delta$ and a complexity function $\epsilon \!\left(n, \delta\right)$, for some hypothesis $h \in \calH$, we have
    \begin{equation*}
    \Prob_{S \sim \calD^n}\qty(\calL_{\calD}(h) \leq \hatL_{S}(h) + \epsilon \!\left(n, \delta\right))\geq 1-\delta.
\end{equation*}
\end{definition}

In this paper, we are mainly interested in two loss functions: the zero-one loss $\ell^{0\textrm{-}1}$ and the cross-entropy loss $\ell^{\textrm{x-e}}$. Given a predictor $h$ and a pair $(\bx,y)$, the zero-one loss is defined as $\ell^{0\textrm{-}1}(h(\bx), y) = \indicator [\argmax_{j \in \{1, \ldots, C\}} h(\bx)_j \neq y]$, where the indicator function $\indicator[a]$ takes the value of $1$ if the predicate $a$ is true and $0$ otherwise.
This loss has its own notation for the true risk and empirical risk, respectively denoted $R_{\calD}(h)$ and $\widehat{R}_{S}(h)$.

The cross-entropy loss is widely used to train neural networks and has recently become an object of interest in statistical learning theory \citep[e.g.,][]{perez2021tighter, lotfi2023non}.
The cross-entropy loss is defined as $\ell^{\textrm{x-e}}(h(\bx), y) = -\ln \left(\sigma(h(\bx), y)\right)$ with the softmax function  
\begin{equation*}
    \sigma(h(\bx), y) = \frac{\exp\left(h(\bx)_y\right)}{\sum_{c=1}^C \exp\left(h(\bx)_c\right)}.
\end{equation*}
We denote $\boldsigma(f(\bx)) = \left(\sigma(f(\bx), 1), \ldots, \sigma(f(\bx), C)\right)$ the softmax vector. By forcing the softmax to always be greater than some positive constant, it is possible to bound the cross-entropy loss, which otherwise tends to infinity when the softmax tends to zero. To do so, we consider the clamped softmax of \citet{perez2021tighter} (see \cref{def:clamped_softmax}) and the smoothed softmax of \citet{lotfi2023non} (see \cref{def:smooth_softmax}). 

We now review the frameworks we leverage to derive the results in \cref{sec:main} and experiment with in \cref{sec:experiments}.

\subsection{Sample Compression Theory}\label{ss:sample_compression}

Introduced by \citet{littlestone1986relating}, sample compression theory provides generalization guarantees for data-dependent predictors that can be fully described by a small subset of the training data. 
Intuitively, if a model's parameters depend on only a few training points, the model is unlikely to have overfit, and sample compression theory formalizes this intuition by providing an upper-bound of the true risk of the model. 
Some examples of algorithms compatible with the sample compression framework include the support vector machine \citep{boser_svm}, the perceptron \citep{rosenblatt1958perceptron, moran2020perceptron}, the decision tree \citep{shah_sample_2007}, the set covering machine \citep{marchand2002set, marchand2003set, marchand2005learning, shah_margin-sparsity_2005} and Pick-To-Learn \citep{paccagnan_pick_learn_2023,marks2025pick, paccagnan2025pick}. The latter was shown to be very effective at deriving tight deep learning bounds \citep{bazinet2024, comeau2025sample}.

We denote the compression set $S_{\bfi} = \{(\bx_{i}, y_i)\}_{i \in \bfi}$, which is defined using a strictly increasing sequence of $\m$ indices $\bfi = (i_1, \ldots, i_{\m})$. All sequences $\bfi$ belong to $\scriptP(n)$, the set of all the $2^n$ strictly increasing sequences composed of the numbers $1$ through $n$. We denote the complement set $S_{\bfi^c} = S \setminus S_{\bfi}$ with $|\mathbf{i}^c| = n-\m$. The subset of sample-compressed predictors is denoted $\calH_{S} \subseteq \calH$.
A predictor \mbox{$h = A(S)$} is called sample-compressed if there exists a function \mbox{$\scriptR : \bigcup_{m \leq n} (\calX \times \calY)^m \to \calH_S$} and a sequence \mbox{$\bfi \in \scriptP(n)$} such that $\scriptR(S_{\bfi})$ would return the same predictor as the original algorithm ($A(S) = \scriptR(S_{\bfi})$). For any such predictor, the sample-compression bound of \citet{bazinet2024} can be used to upper-bound the true risk of the predictor.
\begin{theorem}[\cite{bazinet2024}]\label{thm:squared_bazinet}
For a distribution $\calD$ over $\calX \times \calY$,
a loss $\ell: \R^C \times \calY \to [B_{\ell}, T_{\ell}]$ and $\delta \in (0,1]$, with probability at least $1-\delta$ over the draw of $S \sim \calD^n$, simultaneously for all compression sets $\bfi \in \scriptP(n)$, we have 
\begin{align*}
\mathcal{L}_{\calD}(\scriptR(S_{\bfi})) \leq \hatL_{S_{\bfi^c}}(\scriptR(S_{\bfi})) + \lambda_{\ell}\sqrt{\frac{1}{2|\!\bfi^c\!|}\log\left(\frac{2\sqrt{|\mathbf{i}^c|}}{P_{n}(\m)\delta}\right)},\\[-8mm]
\end{align*}
with $P_n(\m) = \tfrac{6}{\pi^2}(\m+1)^{-2}\smqty(n \\ \m)^{-1}$.
\end{theorem}
This result is versatile and can be applied to any sample-compressed predictor. In practice, one needs to prove that an algorithm produces sample-compressed models in order to compute the bound. A notable example for deep neural networks is the aforementioned Pick-To-Learn algorithm.
Interestingly, the coreset methods (see \citet{moser2025coreset} for an introduction) also fit this framework: a coreset can be viewed as the compression set in \cref{thm:squared_bazinet} to obtain generalization bounds. To the best of our knowledge, the connection between coreset methods and sample compression has never been investigated.

\subsection{Model Compression Theory}\label{ss:model_compression}

The code length of a model is exploited by  \citet{zhou2018nonvacuous} to measure its complexity: 
if it is possible to reduce the code length of a model, by either decreasing the number of trained parameters or quantizing its weights without performance degradation, then the model should generalize well. In a similar fashion, \citet{arora2018stronger} provided computable bounds for compressed models by studying the robustness of the layers to perturbations of an uncompressed model and relating it to the description length of the model. Following \citet{zhou2018nonvacuous}, \cite{lotfi2022pac} presented a tight bound for model compression based on the universal prior \citep{solomonoff1964formal}, which introduces the Kolmogorov complexity \citep{kolmogorov1965three} of the model into the bound. 
In practice, they perform subspace compression \citep{li2018measuring} and adaptative quantization \citep{han2015deep} to minimize the complexity of the model.
Finally, \citet{lotfi2023non} proposed a hybrid approach between subspace compression and LoRA \citep{hu2022lora} named SubLoRA and extended the generalization bound to real-valued losses, a result that we now present.
Let $l_{\scriptC}(\hat{h})$ denote the code length in bits of a model $\hat{h}$ according to a code $\scriptC$,
and $\calH_{\scriptC}\subseteq \calH$ be the subset of compressed predictors that can be encoded by~$\scriptC$.
\begin{theorem}[\cite{lotfi2023non}]\label{thm:lotfi}
For a distribution~$\calD$ over $\calX \times \calY$, a loss $\ell: \R^C \times \calY \to [B_{\ell}, T_{\ell}]$ and $\delta \in (0,1]$, with probability at least $1-\delta$ over the draw of $S \sim \calD^n$, simultaneously for all $\hat{h} \in \calH_{\scriptC}$, we have 
 \begin{align*}
\mathcal{L}_{\calD}(\hat{h}) \leq  \hatL_{S}(\hat{h}) + \lambda_{\ell}\sqrt{\frac{1}{2n}\qty[l_{\scriptC}(\hat{h})^2 \log 2^{l_{\scriptC}(\hat{h})} + \log\frac{1}{\delta} ]}\,.
\end{align*}
\end{theorem}
\subsection{PAC-Bayes Theory}

While the previous sections presented inequalities bounding the true risk of a single model, the PAC-Bayes framework principally studies the true risk of stochastic predictors expressed as a distribution over multiple models. Built on the work of \citet{mcallester1998some} and \citet{shawe1997pac}, PAC-Bayes theory was popularized for stochastic neural networks by 
\citet{dziugaite2017computing} and \citet{perez2021tighter}. 
Recent developments include applications to meta-learning \citep[e.g.,][]{guan2022fast, zakerinia2024more, leblanc2025generalization}, deep learning generative networks \citep[e.g.,][]{cherief2022pac, mbacke2023pac, mbacke2023statistical, mbacke2024a} and large language models \citep{su2024mission}. 

In this setting, one is interested in learning a distribution~$Q$ over the set of predictors $\calH$. Starting with a data-independent prior distribution $P$, an algorithm $A'$ takes as input the dataset~$S$ and the prior distribution $P$ and then outputs the posterior distribution $Q= A'(S, \ P)$. 
PAC-Bayes bounds control generalization by penalizing posteriors that deviate from the prior, as typically measured by the Kullback-Leibler ($\KL$) divergence: $$\KL(Q||P) = \E_{h \sim Q} \ln\frac{Q(h)}{P(h)}.$$
The following PAC-Bayesian theorem was first presented by \citet{mcallester2003pac} and tightened by \citet{maurer2004note}.

\begin{theorem}[\citet{mcallester2003pac}]\label{thm:mcallester}
For a distribution~$\calD$ over $\calX \times \calY$, a data-independent prior distribution $P$ over $\calH$, a loss $\ell: \R^C \times \calY \to [B_{\ell}, T_{\ell}]$ and $\delta \in (0,1]$, with probability at least $1-\delta$ over the draw of $S \sim \calD^n$, simultaneously for all $Q$ over $\calH$, we have
\begin{equation*}
\E_{h \sim Q}\! \calL_{\calD}(h) \leq \E_{h \sim Q}\! \hatL_{S}(h) + \lambda_{\ell}\sqrt{\frac{\KL(Q||P) + \log \frac{2\sqrt{n}}{\delta}}{2n}}.
\end{equation*}
\end{theorem}

\subsection{Other Theoretical Frameworks}

There exist several other approaches in statistical learning theory beyond those discussed above. Notably, norm-based generalization bounds \citep{bartlett2017spectrally} are based on the norm of the network's weight matrices as a measure of complexity.
However, these bounds are generally vacuous \citep{galanti2023normbased}, sometimes negatively correlated with generalization \citep{jiang2019fantastic}, or cannot be computed, as they either require the input space to be bounded \citep{golowich2018size} or depend on intractable universal constants \citep{bartlett2002rademacher,lin2019generalization, long2019generalization,wei2019improved, ledent2021norm,pinto2025on}.
Similarly, VC dimension bounds are intractable or vacuous for deep neural networks \citep{vcdim, bartlett1998almost,bartlett2019nearly}.

In recent years, bounds based on information theory gained interest, as the conditional mutual information framework \citep{steinke2020reasoning, hellstrom2020generalization} tends to yield non-vacuous bounds for stochastic gradient descent \citep{hellstrom2022a}.
However, these bounds either hold in expectation or use a super-sample of size $2n$, where the model is trained on a first half of the super-sample and the guarantee upper-bounds the risk on the second half, instead of the true risk.

Finally, \citet{galanti2023comparative} use the effective depth of the neural network as complexity measure. However, their bound is in expectation and depends on multiple assumptions that are not always satisfied in practice. \citet{than2025non} provide vacuous generalization guarantees for trained models that depend on a partition of the input space.

\section{Disagreement-Based Bounds}
\label{sec:main}

As discussed in the previous section, most bounds for deep learning are either vacuous, not computable or hold only for a restricted class of hypotheses. In this section, we present new theoretical results applicable to any machine learning model.
We first formalize the general ideas underpinning our approach by stating a  
{\sc Bound sketch},
which presents the form of our proposed disagreement-based bounds over the gap between the losses of two models,
and show how this can be leveraged for three use cases.
We then instantiate our general scheme to three different types of losses: the zero-one loss (\cref{thm:bin_disag}), Lipschitz-continuous losses (\cref{thm:softmaxloss}) and non-Lipschitz-continuous losses (\cref{thm:fullloss}), a weaker result which requires a labeled dataset for computing the disagreement.

\label{ss:new_disag_bound}

The following {\sc Bound sketch} provides the general scheme of the forthcoming high-probability bounds on the gap between the true losses of two predictors.
We denote $U = \{(\bx_i, \cdot)\}_{i=1}^m$ an unlabeled dataset sampled \emph{i.i.d.} from the distribution $\calD$.

\ourparagraph{\textsc{Bound sketch}.}%
 {\it   Given two predictors $f, h \in \calH$, a loss function $\ell$, and a proper (to be defined) disagreement measure $D_U(f,h; \ell, \delta)$  between $f$ and $h$ over a dataset $U$. The proposed disagreement-based bounds are such that, with probability at least $1{\,-\,}\delta$ over the sampling of $U {\,\sim\,} \calD^m$, we have}
    \begin{equation} \label{eq:conjecture}
        \calL_{\calD}(f) \leq \calL_{\calD}(h) + D_U(f,h; \ell, \delta).
    \end{equation}
From now on, we take $f$ to be the target model, i.e., a model that does not enjoy tight generalization bounds, either because its complexity is too large or because it doesn't fit the required assumptions, and $h$ to be the surrogate model for which tight generalization bounds exist. 
Let us describe three use cases of the proposed disagreement-based bounds. 

\ourparagraph{Use case \#1 : Certifying the target model.}
Let $\overline{\calH} \subseteq \calH$ be a class of certifiable surrogate models, such as the class of sample-compressed predictors $\calH_S$ or the class of compressed models $\calH_{\scriptC}$. Let $Q$ be a distribution over $\overline{\calH}$. 
Suppose that there exists a function $\epsilon \!\left(n, \delta\right)$ bounding the generalization gap of the  proxy model $h$, that is, the following generalization bound on its true risk holds with probability at least $1-\delta$ over the sampling of $S \sim \calD^n$:
\begin{equation} \label{eq:example_bound}
    \forall h \in \overline{\calH} : \calL_{\calD}(h) \leq \hatL_{S}(h) + \epsilon \!\left(n, Q(h)\delta \right).
\end{equation}
Examples of the function $\epsilon(n, Q(h)\delta)$ would be the square-root terms in Theorems~\ref{thm:squared_bazinet}, \ref{thm:lotfi} and \ref{thm:mcallester}, respectively for surrogate models from sample compression, model compression and PAC-Bayes theory.

Then, given a chosen model $h^{\star} \in \overline{\calH}$, combining \cref{eq:conjecture,eq:example_bound} gives that, with probability at least $1-2\delta$ over the sampling of $S \sim \calD^n$ and $U \sim \calD^m$:
\begin{equation}\label{eq:conjecture_2}
    \calL_{\calD}(f) \leq \hatL_{S}(h^{\star}) + \epsilon \! \left(n,  Q(h^{\star})\delta)\right) + D_U \!\left(f,h^{\star}; \ell, \delta \right).
\end{equation}
Note that this type of bounds does not hold simultaneously for all $h \in \overline{\calH}$, but the model $h^{\star}$ can be dependent on the dataset~$S$, notably via an optimization algorithm. 
Furthermore, the model $f$ can also depend on $S$, as its complexity does not appear in the bound of \cref{eq:conjecture_2}. 
When a surrogate model $h^{\star}$ with tight generalization guarantees is obtainable, then this strategy renders a certificate for $f$ whose tightness depends on the disagreement measured by the chosen function $D_U(f,h; \ell, \delta)$. 

\ourparagraph{Use case \#2 : Minimizing the disagreement.}
The second use case of this disagreement bound is to certify the true risk of the target model whilst minimizing the disagreement between the two models. With probability at least $1-2\delta$ over the sampling of $S \sim \calD^n$ and $U \sim \calD^m$, simultaneously for all $h \in \overline{\calH}$, we have
\begin{equation}\label{eq:conjecture_3}
    \calL_{\calD}(f) \leq \hatL_{S}(h) + \epsilon \!\left(n, Q(h)\delta\right) + D_U\!\left(f,h; \ell, Q(h)\delta\right).
\end{equation}
We obtain this result by applying \cref{eq:conjecture} with a union bound over all $h \in \overline{\calH}$, followed by a union bound with \cref{eq:conjecture_2}. Although this added step loosens the bound in principle, it allows the surrogate model to be selected and trained on both the datasets $S$ and $U$. One way to train the surrogate model on the unlabeled set is via model distillation \citep{schmidhuber, hinton2015distilling}, using pseudo-labels produced by the target model.

\ourparagraph{Use case \#3 : Bounding the true risk gap.}
The final use case of this bound is to guarantee that replacing a model $h$ with a new model $f$ with better qualities does not lead to a significant performance drop. 
Such qualities could be that the model is fairness-aware \citep[e.g.,][]{dwork2012fairness, hardt2016equality}, differentially private \citep[e.g.,][]{dwork2006differential, dwork2014algorithmic} or achieves faster inference, either  via model distillation \citep[e.g.,][]{buciluǎ2006model, hinton2015distilling} or model quantization \citep[e.g.,][]{han2015deep, nagel2020up}. Note that bounding the risk gap can lead to useful insights even when both models are too complex to enjoy non-vacuous generalization bounds.

\subsection{Disagreement With the Zero-One Loss}

We present our first result for the zero-one loss. To do so, we build on the work of \citet{yang2024mind}, who presented a theoretical result to compare the reconstruction error of a coreset.
We show that this result can be upper-bounded with high probability via the test set bound of \citet{langford2005tutorial} (see \cref{thm:langford2005}), which is essentially perfectly tight for binomial distributions. With these two results, we can state a disagreement measure for the zero-one loss, which is defined as
\begin{equation*}
    \dzero_U(f, h) {\,=\,} \frac{1}{m}\sum_{i=1}^{m} \indicator\qty[ \argmax_{j} f(\bx_i)_j\neq \argmax_k h(\bx_i)_k]\! .
\end{equation*}
The disagreement measure $\dzero_U$ compares the labels predicted by each model for a given data point. If the disagreement bound between $f$ and $h$ tends to zero, then the models achieve the same true risk. We now present our first result: a disagreement-based bound for the zero-one loss.
The proof of all results are given in Appendix~\ref{app:proofs}.

\begin{restatable}{theorem}{binthm}\label{thm:bin_disag}
For two predictors $h \in \overline{\calH}$ and $f \in \calH$, with probability at least $1-\delta$ over the sampling of $U \sim \calD^m$, we have
\begin{equation*}
     R_{\calD}(f) \leq R_{\calD}(h) + \overline{\text{\emph{Bin}}}\qty(m \dzero_U(f, h), m, \delta),
\end{equation*}
with $\overline{\emph{Bin}}\qty(k,m, \delta) = \argsup_{p \in [0,1]}\{\emph{Bin}(k, m, p) \geq \delta \}$ and $\emph{Bin}(k, m, p) = \sum_{i=0}^k \smqty(m \\ i) p^i (1-p)^{m-i}$.
\end{restatable}

This result provides generalization guarantees for the zero-one loss of any target model $f$ by training a surrogate predictor $h$ with tight generalization bounds and a small disagreement with $f$; In the experiments of \cref{sec:experiments}, we do so by bounding the true risk $R_{\calD}(h)$ of the surrogate model using sample compression and model compression guarantees.

By definition, the binomial tail inversion $\overline{\text{\textrm{Bin}}}(k, m, \delta)$ \citep{langford2005tutorial} gives the tightest upper-bound on the error rate of a binomial random variable. To do so, the binomial tail inversion returns the greatest error rate $p$ such that $\text{\textrm{Bin}}(k, m, p)$ the probability of making $k$ or fewer errors out of $m$ samples is greater than $\delta$. As it computes the exact worst-case error rate, there cannot be a tighter bound.

\subsection{Disagreement With Lipschitz Losses}

We now generalize the previous result for Lipschitz-continuous bounded losses. 
An intuitive measure of disagreement between two models is to compare the decision boundary of the models by comparing their output distributions : 
\begin{equation*}
    d_U^{K_{\ell}}(f,h) = \frac{1}{m}\sum_{i=1}^m\|\boldsigma(f(\bx_i)) - \boldsigma(h(\bx_i))\|_1\,.
\end{equation*}
When this measure of disagreement tends to zero, the models predict the same class for each data point sampled from~$\calD$, which means that the models generalize similarly. The following \cref{thm:softmaxloss} uses the Chernoff bound for random variables in the interval unit of \citet{foong2022note} (see \cref{thm:chernoff_bound}) in place of the zero-one loss-based test set bound exploited by \cref{thm:bin_disag}. Furthermore, 
\cref{thm:softmaxloss} provides a bound on a target predictor $f$ given a distribution~$Q$ over a class of surrogate models.
This allows us to apply PAC-Bayes bounds (such as \cref{thm:mcallester}) to the true risk of the surrogate model. Still, the following disagreement-based bound for Lipschitz-continuous $[B_{\ell},T_{\ell}]$-losses can be applied to a single surrogate model by setting $Q$ to a Dirac distribution over a single hypothesis.
\begin{restatable}{theorem}{softmaxloss}\label{thm:softmaxloss}
For a distribution $Q$ over $\calH$, a predictor ${f \in \calH}$,  a Lipschitz loss $\ell : \R^C \times \calY \to [B_{\ell},T_{\ell}]$ with a Lipschitz constant $K_{\ell}$, with probability at least $1-\delta$ over the sampling of $U \sim \calD^m$, we have
\begin{equation*}
    \calL_{\calD}(f) \leq \E_{h \sim Q} \calL_{\calD}(h) + 2K_{\ell}\kl^{-1}\qty(\E_{h \sim Q} \frac{d_U^{K_{\ell}}(f,h)}{2}, \frac{\log\frac{1}{\delta}}{m})
\end{equation*}
\iflong
with $\kl^{-1}(q,\epsilon) = \argsup_{p \in [0,1]} \left\{\kl(q,p) \leq \epsilon\right\}$ and $\kl(q,p) = q \log\tfrac{q}{p} + (1-q)\log\tfrac{1-q}{1-p}$. \else
with $\kl^{-1}(q,\epsilon) = \argsup_{p \in [0,1]} \left\{\kl(q,p) \leq \epsilon\right\}$ and $\kl(q,p) = q \log\tfrac{q}{p} + (1-q)\log\tfrac{1-q}{1-p}$.\fi
\end{restatable}
The above result can be applied to common losses such as the logistic loss, the hinge loss and the Huber loss, which are known to be Lipschitz functions \citep{chinot2018statistical}. The cross-entropy loss is Lipschitz with respect to the softmax output of the model when constraining the entries of the softmax output to be greater than a given positive constant, which is achieved with the clamped or the smoothed softmax. However, the obtained Lipschitz constant $K_{\ell}$ is typically large and renders vacuous bounds. In the following subsection, we present a weaker result that bounds the loss of a model using disagreement on a labeled dataset.

\begin{table*}[!h]
    \begin{minipage}{\columnwidth}
        \centering
    \caption{Generalization bounds for the zero-one loss of the target models according to the approach (and theorem) used to certify the surrogate. The bound values are reported in $\%$. 
    }
    \begin{adjustbox}{width=\columnwidth}
    \begin{tabular}{llcc}
    \toprule
        & Bound (Theorem) & MNIST & CIFAR10\\ \midrule 
        \multirow{ 2}{*}{Baselines}
        & Partition-based (\ref{thm:partition}) & 86.55$\pm$0.53 & 90.58$\pm$0.47\\
        & Norm-based (\ref{thm:norm}) & \small{(3.14$\pm$0.37)$\times 10^{8}$}
        & \small{(3.80$\pm$0.30)$\times 10^{21}$}\\\midrule
        \multirow{ 5}{*}{Our approach}
        &Random Coreset (\ref{thm:test_set_data}) & \textbf{1.79$\pm$0.19} & \textbf{17.36$\pm$0.20}\\
        &Best Coreset (\ref{thm:binom_tail}) & 21.50$\pm$0.83 & 81.72$\pm$0.69\\
        &Pick-To-Learn (\ref{thm:p2l}) & 3.90$\pm$0.09 & 57.13$\pm$2.71\\
        &Model compression (\ref{thm:lotfi_bin}) & 3.45$\pm$0.11 & 35.06$\pm$0.40\\
        & PAC-Bayes (\ref{thm:pbb}) & 4.83$\pm$0.32& 28.13$\pm$1.55\\ 
         \bottomrule
    \end{tabular}
    \label{tab:summary_risk}
    \end{adjustbox}
    \end{minipage}
\hfill
    \begin{minipage}{\columnwidth}
        \centering
    \caption{Generalization bounds for the smoothed cross-entropy loss of the target models according to the approach (and theorem) used to certify the surrogate.
    }
    \begin{adjustbox}{width=\columnwidth}
    \begin{tabular}{llcc}
    \toprule
         &Bound (Theorem) & MNIST & CIFAR10\\ \midrule 
        \multirow{ 2}{*}{Baselines}
        &Partition-based (\ref{thm:partition}) & 7.9546$\pm$0.0475 & 8.3437$\pm$0.0431\\
        & Norm-based & N/A & N/A \\\midrule
        \multirow{ 5}{*}{Our approach}
        &Random Coreset (\ref{thm:test_set_chernoff_bound}) & \textbf{0.0744$\pm$0.0026} & \textbf{0.6858$\pm$0.0089}\\
        &Best Coreset (\ref{thm:bazinet}) & 0.9035$\pm$0.0531 & 5.7076$\pm$0.0970\\
        &Pick-To-Learn (\ref{thm:bazinet}) & 1.2673$\pm$0.0248  & 7.4941$\pm$0.2277 \\
        &Model compression (\ref{thm:lotfi_kl}) & 0.1756$\pm$0.0041 &  1.7851$\pm$0.0223\\
        &PAC-Bayes (\ref{thm:pbb}) & 0.2592$\pm$0.0087& 1.4204$\pm$0.1377\\
         \bottomrule
    \end{tabular}
    \label{tab:summary_loss}
    \end{adjustbox}
    \end{minipage}
\end{table*}

\subsection{Disagreement With Non-Lipschitz Losses}

Let $L = \{(\bx_i, y_i)\}_{i=1}^m \sim \calD^m$ be a small held-out labeled set of data points removed before training $f$. Then, we provide the following disagreement bound for any bounded non-Lipschitz loss. 
\begin{restatable}{theorem}{fullloss}\label{thm:fullloss}
    For a distribution $Q$ over $\calH$, a predictor $f \in \calH$, a loss $\ell : \R^C \times \calY \to [B_{\ell},T_{\ell}]$ with range $\lambda_{\ell} = T_{\ell} - B_{\ell}$, with probability at least $1-\delta$ over the sampling of $L \sim \calD^m$:
    \begin{align*}
        \calL_{\calD}(f) &\leq \E_{h \sim Q} \calL_{\calD}(h) + \lambda_{\ell}\kl^{-1}\qty( \E_{h \sim Q} \frac{d_L(f,h)}{\lambda_{\ell}}, \frac{\log\frac{1}{\delta}}{m}),
    \end{align*}
    with 
    $$d_L(f,h) = \frac{1}{m}\sum_{i=1}^m \left| \ell(f(\bx_i), y_i) -\ell(h(\bx_i),y_i)\right|.$$
\end{restatable}

In contrast with \cref{thm:bin_disag,thm:softmaxloss}, the disagreement $d_L$ is not a disagreement between the output of the models.
However, we show in Corollary~\ref{corr:lip} that, for Lipschitz losses with large~$K_{\ell}$ constants, $d_L$ is a proxy for disagreement between the models, as it is upper-bounded by the disagreement between the probability distributions outputted by the models. We instantiate this result for the clamped and smoothed cross-entropy respectively in \cref{ss:lipschitz}.

For both Theorems~\ref{thm:softmaxloss} and \ref{thm:fullloss}, the expectation of the true loss of the surrogate might not be tractable, which is the case for stochastic neural networks 
\citep{perez2021tighter}. Following the Monte Carlo sampling approach of \citet{perez2021tighter}, we use a Chernoff bound \citep{foong2022note} to upper-bound the expectation of both the risks of the surrogate and the disagreement between the surrogate and target models, as described in \cref{app:ss:pb_disag}.

When used with the zero-one loss, $d_L$ is upper-bounded by~$\dzero_U$, leading to a bound on the true risk of the target using PAC-Bayes theory and an unlabeled dataset $U$, which wasn't possible with the test set bound of \citet{langford2005tutorial}.

\section{Experiments}\label{sec:experiments}

\begin{figure*}[!t]
    \centering
    \begin{subfigure}[b]{0.49\textwidth}
        \includegraphics[width=1\linewidth]{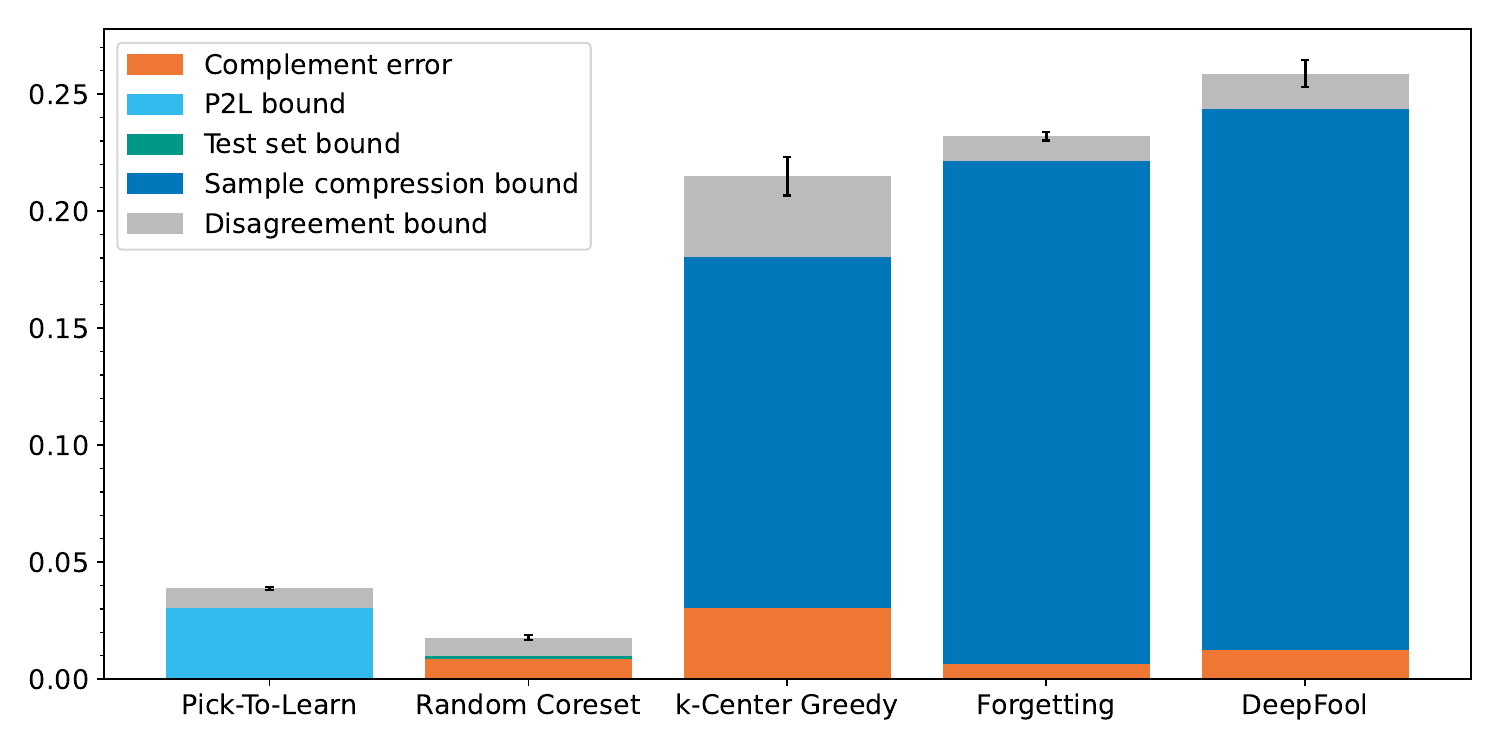}
        \caption{Zero-one loss}
    \end{subfigure}
    \begin{subfigure}[b]{0.49\textwidth}
        \includegraphics[width=1\linewidth]{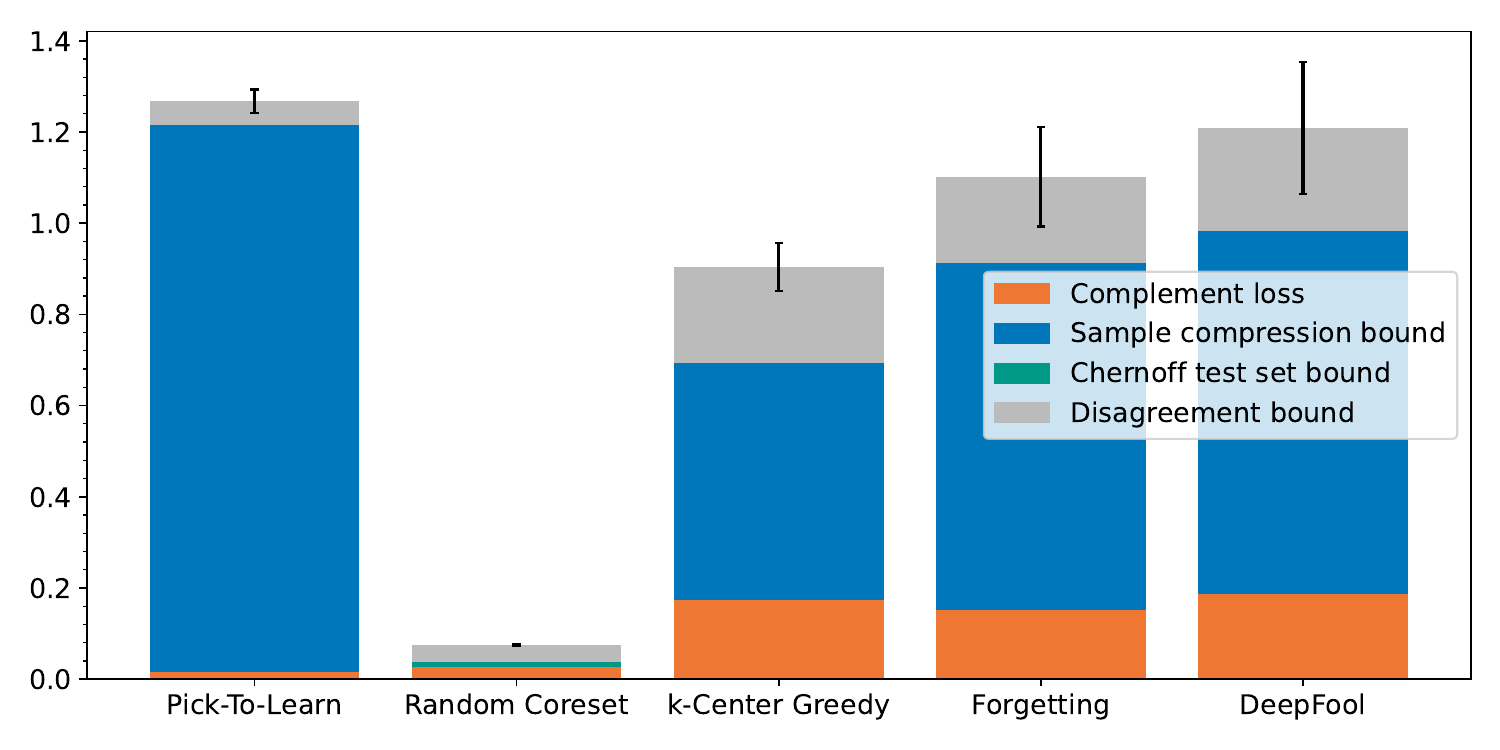}
        \caption{Smoothed cross-entropy loss}
    \end{subfigure}
    \caption{Illustration of the behavior of our disagreement bounds on MNIST using sample-compression methods.}
    \label{fig:bar_graph_mnist}
\end{figure*}

In this section, we demonstrate the versatility and efficacy of our framework by training and certifying a variety of deep neural networks for which tight generalization bounds were previously unavailable.\footnote{Our code is available at \url{https://github.com/GRAAL-Research/Bound-to-Disagree}.} 
In \cref{sec:first_sec_of_exp}, we focus on use case~\#1 and thus leverage \cref{thm:bin_disag} to bound the zero-one loss and \cref{thm:fullloss} to bound the cross-entropy loss of the target network, obtaining a surrogate via sample compression, model compression or PAC-Bayes training. 
In \cref{ss:exp_distillation}, we illustrate use case \#2 by applying \cref{thm:softmaxloss} to the true Huber loss of the target model, while optimizing a compressed surrogate via model distillation. 
Finally, in \cref{ss:exp_gap}, we consider use case \#3 and evaluate the certificate on the gap between the true risks of the target model and its quantized version.

For all experiments, we report the mean and standard deviation over five training-evaluation runs. 
We randomly sample 10\% of the training set to form a validation set.
Of the remaining training set, we further select a disagreement set, corresponding to 20\% of data for MNIST and CIFAR10, and 85\% for Amazon polarity \footnote{15\% corresponds to roughly 500000 data points, more than enough data to train the models on Amazon polarity.}. 
We optimize the smoothed softmax with $\alpha=10^{-3}$, following the result of an ablation study on coreset methods in \cref{sss:ablation}. 
Thus, in all tables and figures, the cross-entropy loss is bounded by $\ln(10^3 C)$, which is approximately $9.21$ with $C=10$. All the bounds presented hold with probability $1-\delta =0.99$. Other hyperparameters are detailed in \cref{app:experiments}.

\subsection{Certifying the Target Model}\label{sec:first_sec_of_exp}

For these experiments, we first train a CNN on MNIST \citep{lecun1998_mnist} and a ResNet18 \citep{he2016deep} on \mbox{CIFAR10} \citep{cifar10} as target models. 
Because these models are deterministic, trained on the whole training set and not quantized, the only generalization bounds of the literature that we can compare with are: 
the \textbf{partition-based bound} of \citet{than2025non} (see \cref{thm:partition}), and \textbf{norm-based bounds}, among which we present the results of \citet{galanti2023normbased} (see \cref{thm:norm}), which is only defined for the zero-one loss. 
Note that for the partition-based bound, contrary to the original paper, we compute the partition on the disagreement set, to obtain non-vacuous results. See \cref{app:partition} for an ablation study on the way of computing the partition. 
Tables~\ref{tab:summary_risk} and~\ref{tab:summary_loss} report a summary of the value of our new disagreement bounds for all the experiments, where we specify  the generalization framework from the literature used to certify the surrogate model. The disagreement bounds displayed were obtained without considering the disagreement by choosing the surrogates with the tightest generalization bounds. 
We observe that both norm-based bounds and partition-based bounds are unsuited to certify the target model both on MNIST and on CIFAR10: 
norm-based bounds are vacuous as cross-entropy maximizes the network's weight norms \citep{galanti2023normbased},
and partition-based bounds are almost trivial, as they estimate a generalization performance close to a random predictor (90\%). Moreover, the latter cannot differentiate between two problems of varying complexities, such as MNIST and CIFAR10.

\begin{table*}[!t]
    \begin{minipage}{\columnwidth}
    \centering
    \caption{Generalization bounds on the zero-one loss achieved on MNIST and CIFAR10 using model compression (MC) bounds. We present the metrics for the surrogate model used to certify the target model. All metrics are in percents (\%), except the size.} 
    \begin{adjustbox}{width=\columnwidth}
    \begin{tabular}{lccccccc}
    \toprule
        \multirow{2}{3em}{Dataset} & \multicolumn{3}{c}{Surrogate model} & \multicolumn{2}{c}{Target model} 
        \\ \cmidrule(l{2pt}r{2pt}){2-4} \cmidrule(l{2pt}r{2pt}){5-6}
        & Test error & Size (KB) & MC Bound & Test error & Our bound \\ \midrule 
MNIST\!\!\!\! &  0.89$\pm$0.08 & 0.04$\pm$0.00 & 2.45$\pm$0.13 & 0.50$\pm$0.07 & 3.45$\pm$0.11 \\
CIFAR10\!\!\!\!\!\! & 14.95$\pm$0.42 & 0.03$\pm$0.00 & 20.15$\pm$0.35 & 5.84$\pm$0.16 & 35.06$\pm$0.40 \\
         \bottomrule
    \end{tabular}
    \label{tab:model_comp_risk}
    \end{adjustbox}
    \end{minipage}
\hfill
    \begin{minipage}{\columnwidth}
    \centering
    \caption{Generalization bounds on the zero-one loss achieved on MNIST and CIFAR10 using PAC-Bayes (PB) bounds. We present the metrics for the surrogate model used to certify the target model. All metrics presented are in percents (\%), except the $\KL$.}
    \begin{adjustbox}{width=\columnwidth}
    \begin{tabular}{lcccccccc}
    \toprule
        \multirow{2}{3em}{Dataset} & \multicolumn{3}{c}{Surrogate model} & \multicolumn{2}{c}{Target model}  \\ \cmidrule(l{2pt}r{2pt}){2-4} \cmidrule(l{2pt}r{2pt}){5-6}
        & Test error & $\KL$ & PB Bound & Test error & Our bound\\ \midrule 
         MNIST\!\!\!\! & 0.82$\pm$0.07 & 4.68$\pm$0.22 & 2.28$\pm$0.13 &   0.50$\pm$0.07  & 4.83$\pm$0.32 \\
CIFAR10\!\!\!\!\!\!  & 11.33$\pm$0.55 & 1.62$\pm$0.92 & 14.04$\pm$0.91 & 5.84$\pm$0.16 & 28.13$\pm$1.55 \\
         \bottomrule
    \end{tabular}
    \label{tab:pbb_risk}
    \end{adjustbox}
    \end{minipage}
\end{table*}

\ourparagraph{Sample compression}
In our first experiment, we use sample-compression bounds to certify the target models on MNIST and CIFAR10. To do so, we train neural networks with the same architecture as the targets using Pick-To-Learn (P2L) \citep{paccagnan_pick_learn_2023} and coreset methods such as Random Coreset, k-Center Greedy \citep{sener2017active}, Forgetting \citep{Forgetting} and DeepFool \citep{ducoffe2018adversarial}.

\cref{fig:bar_graph_mnist} presents the contribution of each term of the bound on MNIST for the different sample-compression methods. The results for CIFAR10 are presented in \cref{fig:bar_graph_cifar10} in \cref{app:experiments}. 
The complement error is added as complementary information, but the bound is not linear in the complement error.
We observe that each method is able to achieve tight generalization bounds for our target models. The tightest bound is achieved by the Random Coreset approach, which is expected as it enjoys a test set bound, which means that the complexity of the random coreset is never accounted for in the bound. In comparison, the other bounds are train set bounds, meaning the coreset method must find a compromise between accuracy and coreset size. Moreover, the Random Coreset approach was found to be a very competitive baseline by \citet{guo2022deepcore}, which can explain that this method achieves a small train and test error.
For the zero-one loss, P2L achieves the second-tightest bound. For the cross-entropy loss, it actually achieves the worst bound of all methods. The complexity of the model learned using P2L is much higher than the coreset methods, however, the P2L bound (\cref{thm:p2l}) still achieves significantly tighter bounds. In \cref{tab:summary_risk,tab:summary_loss}, we report the bound for P2L, Random Coreset and the best (non-random) coreset method. Of all the methods presented, Random Coreset achieves the tightest disagreement bound for our target models. All of the sample compression bounds are significantly tighter than both the norm-based bound and the partition-based bound.

\ourparagraph{Model compression.}
For the model compression experiments, we use the SubLoRA method of \citet{lotfi2023non}. To showcase a different approach that leads to tight certificates, we start by pretraining a CNN on 20\% of the training set for MNIST and a ResNet18 on 50\% of the training set for CIFAR10. We then add a low-rank adapter \citep{hu2022lora} and use the subspace compression method \citep{li2018measuring} implemented by \citet{lotfi2022pac} on the adapter. After training the adapter on the remaining training data, the subspace vector is quantized using adaptative quantization \citep{han2015deep} and encoded using arithmetic encoding \citep{arithemticcoding}. The bound on the surrogate's loss is computed over the set on which the adapter was trained.

In \cref{tab:model_comp_risk}, we report the bound on the zero-one loss for the SubLoRA approach on MNIST and CIFAR10. We report the bound on the cross-entropy loss in \cref{tab:model_comp_loss}. On MNIST, the model compression bounds are the second-tightest bounds reported in \cref{tab:summary_risk}, whilst they are the third-tightest bounds for CIFAR10. It was expected that these bounds would be much tighter than the sample-compression bounds, as pretraining the models tightens significantly the guarantees \citep{ambroladze2006tighter,parrado2012pac, dziugaite_roy_2018}. However, we note that even after an extensive hyperparameter search on CIFAR10, the bound favors a smaller subspace vector, where the model only marginally improves after the pretraining, if at all.

\ourparagraph{PAC-Bayes.}
For the PAC-Bayes experiments, we train stochastic neural networks with the implementation of \citet{perez2021tighter}. We start by pretraining a deterministic neural network, either a CNN on 20\% of the training set for MNIST or a ResNet18 on 60\% or 70\% on the training set of CIFAR10. We then add Gaussian distributions over the weights of the neural network, with the mean and the standard deviation of the distributions as trainable parameters. To compute the bound of \cref{thm:pbb}, we approximate the expectation via Monte Carlo sampling. The full statement of the disagreement loss with the Monte Carlo sampling is presented in \cref{thm:pbb_disag}.

In \cref{tab:pbb_risk}, we report the bound on the zero-one loss for the stochastic neural networks on MNIST and CIFAR10. The results on the cross-entropy loss are reported in \cref{tab:pbb_loss}. As the models were pretrained, they achieve small test errors even with a small $\KL$ divergence, similarly to the model compression bound. In contrast with the previous approaches, the disagreement bounds are actually larger than the compressed model bounds. The approximation via the Monte Carlo sampling and the variance in the stochastic predictor leads to a larger disagreement between the model and its surrogate. However, this approach provides tight generalization bounds and achieves the second-best disagreement bound on CIFAR10. 

\subsection{Minimizing the Disagreement}\label{ss:exp_distillation}
For the second use case of our disagreement bounds, we provide model distillation experiments on Amazon polarity dataset \citep{zhang2015character}. Our target models are a DistilBERT \citep{sanh2019distilbert} and a GPT2 \citep{radford2019language}. Using the target models, we create pseudo-labels for the unlabeled disagreement set and train the surrogate models on both the labeled set $S$ and the unlabeled set $U$. We freeze the pretrained weights of the model (from the Transformers library \citep{wolf2020huggingface}) and add SubLoRA weights \citep{lotfi2023non} on the top half of the model. As the model is trained on both datasets, the bound holds simultaneously for all possible models, following \cref{eq:conjecture_3}. We use the Huber loss (see \cref{def:huber_loss}) with $\delta_{\textrm{H}}=0.2$ to compute \cref{thm:softmaxloss}, which means that we can train the model with vanilla softmax. 

We present the results for the Huber loss in \cref{tab:amazon_distillation_huber}, where both models achieved non-vacuous generalization bounds. The DistilBERT surrogate achieved a test loss very similar to the one of the target model, while the GPT2 surrogate's loss is much higher. This could be explained by the fact that the model is twice as big, although the SubLoRA adapter that achieved the best disagreement bound is the same size (in KB) as the one of DistilBERT. Indeed, as the distribution $Q$ is present twice in \cref{eq:conjecture_3}, the size of the model becomes much more important in the bound minimization than the loss of the model, leading to higher train and test losses. Although the bound is only available for the Huber loss instead of the cross-entropy loss, the model's softmax doesn't have to be modified, the dataset does not need to be labeled and the bound minimizes the actual disagreement between the two models.
We report the results for the zero-one loss in \cref{tab:amazon_distillation}. For both architectures, we are able to provide non-vacuous and non-trivial (less than 50\% for binary classification) generalization bounds for the targets.

\begin{table}[!t]
    \centering
\caption{Generalization bounds on the Huber loss via model distillation on Amazon polarity using model compression bounds. With $\delta_{\textrm{H}} \,{=}\,0.2$, the loss is less than $\delta_{\textrm{H}} - \tfrac{1}{2}\delta_{\textrm{H}}^2 {\,=\,} 0.18$.
    }
    \begin{adjustbox}{width=\columnwidth}
    \begin{tabular}{lcccccccc}
    \toprule
        \multirow{2}{4em}{Dataset} & \multicolumn{3}{c}{Surrogate model} & \multicolumn{2}{c}{Target model} \\ \cmidrule(l{2pt}r{2pt}){2-4} \cmidrule(l{2pt}r{2pt}){5-6}
        & Test loss & Size (KB) & MC Bound & Test loss & Our bound\\ \midrule 
DistilBERT\!\!\!\! & .015$\pm$.000 & 2.02$\pm$0.01 & .027$\pm$.000 & .012$\pm$.000 & .059$\pm$.001 \\
GPT2 &  .035$\pm$.004 & 2.04$\pm$0.43 & .052$\pm$.004 & .010$\pm$.000 & .152$\pm$.015\\
         \bottomrule
    \end{tabular}
    \label{tab:amazon_distillation_huber}
    \end{adjustbox}
\end{table}

\subsection{Bounding the True Risk Gap}\label{ss:exp_gap}

To demonstrate our third use case, we provide quantization experiments on Amazon polarity, with the same target models as the previous section. We quantize the models and compare their disagreement with the unquantized target models.
In contrast to the previous sections, the surrogates are still too large to enjoy non-vacuous model compression bounds, even after quantization.
We can use our disagreement bounds to provide an upper-bound on the gap between the true risk of two models. This will guarantee that although the surrogate is quantized, the models will perform similarly on unseen data points. We present results for models with and without quantization-aware (QA) training.  For QA training, we quantize to 4 bits using TorchAO \citep{or2025torchao} and 8 bits using AO from PyTorch \citep{Ansel_PyTorch_2_Faster_2024}.
For the other experiments, we use half-quadratic quantization (HQQ) \citep{badri2023hqq} to quantize the models.

We report the results in \cref{fig:quantization}.
Without a surprise, fully quantizing a model to 2 bits leads to large performance loss. The tightest disagreement bounds were obtained by using 8 bits with HQQ for both models. 
The QA training with 4 bits achieves a better test error for both models than 8 bit training, leading, however, the DistilBERT model to overfit on the training set. For both DistilBERT and GPT2, respectively without and with quantization-aware training, the 4-bits models are able to achieve disagreement bound of about 2\%, whilst achieving faster inference and using less memory.

\ourparagraph{Remark.} Finding a good surrogate model is the most important step to use our disagreement framework. We provide some recommendations derived from our experience in training the models presented in this section. In general, choosing the same model architecture for both the target and surrogate models leads to very small disagreement between the models, although choosing a smaller model with model compression theory may yield tighter bounds. For sample compression surrogates, we recommend using the random coreset approach or Pick-To-Learn. The random coreset achieved the tightest bound by far with the smallest computational cost, while Pick-To-Learn gives control on the complement error via early stopping. For model compression surrogates, we found that using SubLoRA \citep{lotfi2023non} with a base model pretrained on 20 to 50\% of the dataset works best. Finally, stochastic neural networks with Gaussian distributions are generally harder to train than the other surrogates, but can yield tight generalization bounds when the base model is pretrained on 50 to 70\% of the data.

\begin{figure}
    \centering
    \begin{subfigure}[b]{0.49\columnwidth}
        \includegraphics[width=1\linewidth]{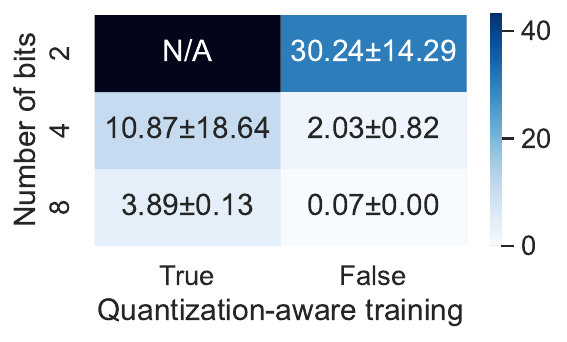}
        \caption{DistilBERT}
    \end{subfigure}
    \begin{subfigure}[b]{0.49\columnwidth}
        \includegraphics[width=1\linewidth]{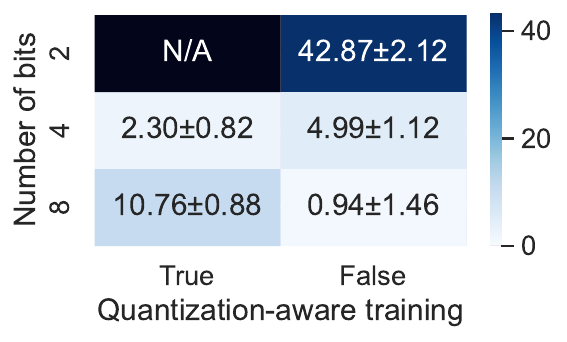}
        \caption{GPT2}
    \end{subfigure}
    \caption{Behavior of the disagreement bound according to the number of bits and the use of QA training.}
    \label{fig:quantization}
\end{figure}

\section{Conclusion}

We proposed new disagreement-based bounds that can be leveraged for any machine learning predictor. These bounds require no assumptions about the model's architecture or the training method. We showed that these bounds can be used with multiple statistical learning theory frameworks to obtain tight generalization bounds, by experimenting on datasets of varying complexity such as MNIST, CIFAR10 and Amazon polarity. 

Despite achieving eloquent improvement over existing bounds in many contexts, our proposed framework also has limitations that deserve to be addressed in future work. Notably, the disagreement-based bounds are only as good as the guarantees used to certify the surrogate models, and finding a surrogate for a very complex model can be challenging. Moreover, the tightest bounds are obtained without minimizing the disagreement on the unlabeled set, giving no control on the disagreement between the target and the surrogate. In particular, the random coreset approach provides little control over the resulting bound, yet yields the tightest bounds of all the approaches. Finally, we believe that this approach may be generalized for regression problems and even possibly to text generation by leveraging the token-level bound of \citet{lotfi2025unlocking}.

\begin{acknowledgements}
        We would like to thank the anonymous reviewers for their helpful comments. We wish to thank Jacob Comeau and Benjamin Leblanc for helping review the paper, Benjamin Guedj, Paul Viallard, Romaric Gaudel and Omar Rivasplata for the insightful discussions that helped elevate this paper and Dario Paccagnan for pointing an error in the P2L bound code before the submission. 
    
        Mathieu Bazinet is supported by a FRQNT B2X scholarship (343192, \url{https://doi.org/10.69777/343192}). Pascal Germain is supported by the NSERC Discovery grant RGPIN-2020-07223. 
\end{acknowledgements}

\bibliography{bibliography}

\newpage
\onecolumn

\title{\mytitle{} \\ (Supplementary Materials)}
\maketitle
\appendix

\longtrue

\setcounter{theorem}{0}
\renewcommand{\thetheorem}{S\arabic{theorem}}
\renewcommand{\thelemma}{S\arabic{theorem}}
\renewcommand{\thecorollary}{S\arabic{theorem}}
\renewcommand{\theproposition}{S\arabic{theorem}}
\renewcommand{\thedefinition}{S\arabic{theorem}}

\setcounter{figure}{0}
\setcounter{table}{0}
\renewcommand{\thefigure}{S\arabic{figure}} 
\renewcommand{\thetable}{S\arabic{table}} 

\section{Definitions and Theoretical Results From the Literature}

\begin{definition}[\cite{perez2021tighter}]\label{def:clamped_softmax} Given a parameter $\alpha \in (0,1)$ and a number of classes $C$, the clamped softmax is defined as
    \begin{equation*}
    \sigma_{\max}(h(\bx),y) = \max\left(\frac{\alpha}{C}, \frac{\exp\left(h(\bx)_y\right)}{\sum_{c=1}^C \exp(h(\bx)_c)}\right),
\end{equation*}
and the cross-entropy loss is bounded with $B_{\ell} = 0$, $T_{\ell} = \ln\frac{C}{\alpha}$ and $\lambda_{\ell} = \ln\frac{C}{\alpha}$.
\end{definition}

\begin{definition}[\cite{lotfi2023non}]\label{def:smooth_softmax}
Given a parameter $\alpha \in (0,1)$ and a number of classes $C$, the smoothed softmax is defined as
    \begin{equation*}
    \sigma_{\emph{smooth}}(h(\bx),y) =  (1-\alpha)\left(\frac{\exp\left(h(\bx)_y\right)}{\sum_{c=1}^C \exp(h(\bx)_c)}\right) + \frac{\alpha}{C}\,,
\end{equation*}
and the cross-entropy loss bounded with $B_{\ell} = \ln \left(1-\alpha + \frac{\alpha}{C}\right)$, $T_{\ell} = \ln\frac{C}{\alpha}$ and $\lambda_{\ell} = \ln \left(1 + (1-\alpha)\frac{C}{\alpha}\right)$.
\end{definition}

\begin{definition}\label{def:huber_loss}
Given $h : \calX \to \R^C$, $y \in \R^C$ and $\delta_{\textrm{H}} > 0$, the Huber loss is defined as 
\begin{align}
    \ell^{\mathrm{Huber}}(h(\bx), y) = \frac{1}{C}\sum_{i=1}^C\begin{cases}
        \frac{1}{2}(\sigma(h'(\bx), i) - y_i)^2 & \text{ if } |\sigma(h'(\bx), i) - y_i| \leq \delta_{\textrm{H}} \\
        \delta_{\textrm{H}}|\sigma(h'(\bx), i) - y_i| - \frac{1}{2}\delta^2_{\textrm{H}} & \text{ otherwise.}
    \end{cases}
\end{align}

As we apply a softmax over the output of the predictor, this loss is bounded by $\delta_{\textrm{H}} - \frac{1}{2}\delta^2_{\textrm{H}}$. This loss is Lipschitz-continuous with $K^{\ell} = \delta_{\textrm{H}}$ \citep{chinot2018statistical}.
\end{definition}

\begin{theorem}[Partition-based bound of \citet{than2025non}]\label{thm:partition}
Let $\Zcal = \calX \times \calY$ and let $\Gamma(\Zcal) = \bigcup_{i=1}^K \Zcal_i$ be a partition of $\Zcal$ into $K$ subsets. Let $n_i = |S \cap \Zcal_i|$ be the number of samples of a dataset $S$ grouped into $\Zcal_i$. Let $\Tcal = \sum_{i=1}^k \indicator[n_i \neq 0]$ be the number of subsets of $\Gamma(\Zcal)$ in which data points from $S$ fall into. For any partition $\Gamma$ into $K$ subsets, for any loss $\ell : \R^C \times \calY \to [0, T_{\ell}]$, for any constants $\gamma \geq 1$ and $\eta \in \left[0, \frac{\gamma n(K + \gamma n)}{K(4n-3)}\right]$, with probability at least $1-\gamma^{-\eta}- \delta$ over the draw of $S \sim \calD^n$, for a given model $h_S$ trained on $S$, we have
    \begin{align*}
        \calL_{\calD}(h_S) \leq \hatL_{S}(h_S) &+ T_{\ell} \sqrt{\eta \ln \gamma}\sqrt{\frac{\gamma}{2n} + \frac{\gamma^2}{2}\sum_{i=1}^K\left(\frac{n_i}{n}\right)^2 + \gamma^2 \sqrt{\frac{2}{n}\ln \frac{2K}{\delta}}} + T_{\ell}(\sqrt{2}+1) \sqrt{\frac{\Tcal\ln\frac{4K}{\delta}}{n}} + \frac{2 T_{\ell}\Tcal\ln\frac{4K}{\delta}}{n}.
    \end{align*}
\end{theorem}

\begin{theorem}[Norm-based bound of \citet{galanti2023normbased}]\label{thm:norm} Given a fixed neural network architecture $G$ that defines a directed acyclic graph. Let the architecture have $L$ layers and $d_l$ neurons at each layer $l \in \{1, \ldots, L\}$. Denote $z_i^l$ is the i-th neuron of the l-th layer. Let $z_i^{l-1}$ be a predecessor of the neuron $z_j^l$ if there exists an edge between the two neurons and denote $pred(l,j)$ the set of predecessors of the neuron $z_j^l$. Let $\calH_G$ be the class of neural networks with architecture $G$. Denote $\rho(h)$ the norm of the model $h \in \calH_G$. Given images with $c_0$ channels and $d_0$ values such that the images are in $\R^{c_0 d_0}$ and labels in $\calY \subset \R^C$ the set of $C$-dimensional one-hot vectors. Let $z_{j}^0(x)$ be the $j^{th}$ pixel of the image. For any distribution $\calD$ over $\R^{c_0 d_0} \times \{-1, +1\}$, with probability at least $1-\delta$ over the draw of $S \sim \calD^n$, we have
    \begin{align*}
        \forall h \in \calH_G : R_{\calD}(h) &\leq \widehat{R}_S^{\, \gamma}(h) + \frac{2\sqrt{2}(\rho(h) + 1)}{\gamma n} \cdot \Lambda(h) + 3\sqrt{\frac{\log(2(\rho(h)+2)^2)/\delta}{2n}},
    \end{align*}
with $\widehat{R}_S^{\, \gamma}(h) = \frac{1}{n}\sum_{i=1}^n \indicator[\max_{j\neq y_i} h(\bx_i)_j +\gamma \geq h(\bx_i)_{y_i}]$ and 
\begin{equation*}
    \Lambda(h) = \left( 1 + \sqrt{2(L \log2 + \sum_{l=1}^{L-1}\log(\max_{j \in [d_l]}|pred(l, j)|) + \log C)}\right) \cdot \sqrt{
        \max_{j_0, \ldots, j_L}\prod_{l=1}^{L-1}|pred(l, j_l)| \cdot \sum_{i=1}^n \|z_{j_0}^0(x_i)\|_2^2}.
\end{equation*}
\end{theorem}

\begin{theorem}[Test set bound of \citet{langford2005tutorial}]\label{thm:langford2005}
Let $X_1, \ldots, X_n$ be \emph{i.i.d.} random variables $X_i \in \{0,1\}$ and $p = \E[X_i]$. Then, with probability at least $1-\delta$, we have
\begin{equation*}
    p \leq  \overline{\emph{Bin}}\left(\sum_{i=1}^n X_i, n, \delta\right),
\end{equation*}
with $\emph{Bin}(k, m, p) = \sum_{i=0}^k \smqty(m \\ i) p^i (1-p)^{m-i}$ and $\overline{\emph{Bin}}\qty(k,m, \delta) = \argsup_{p \in [0,1]}\{\emph{Bin}(k, m, p) \geq \delta \}.$
\end{theorem} 

\begin{theorem}[Chernoff bound of \cite{foong2022note}]\label{thm:chernoff_bound}
    Let $X_1, \ldots, X_n$ be \emph{i.i.d.} random variables $X_i \in [0,1]$ and $p = \E[X_i]$. Then, with probability at least $1-\delta$, we have
    \begin{equation*}
        p \leq \kl^{-1}\qty(\frac{1}{n} \sum_{i=1}^n X_i, \frac{1}{n}\log\frac{1}{\delta}).
    \end{equation*}
    with $\kl(q,p) = q \log\tfrac{q}{p} + (1-q)\log\tfrac{1-q}{1-p}$ and $\kl^{-1}(q,\epsilon) = \argsup_{p \in [0,1]} \left\{\kl(q,p) \leq \epsilon\right\}.$
\end{theorem}

\begin{theorem}[Test set bound of \citet{langford2005tutorial}, adapted for random coresets]\label{thm:test_set_data}
For any distribution $\calD$ over $\calX \times \calY$, for any predictor $h \in \calH$, for any random sequence $\bfi \in \scriptP(n)$ and for any $\delta \in (0,1]$, with probability at least $1-\delta$ over the draw of $S \sim \calD^n$, we have
    \begin{equation*}
        R_{\calD}(A(S_{\bfi})) \leq \overline{\emph{Bin}}\qty((n-\m) \widehat{R}_{S_{\bfi^c}}(A(S_{\bfi})), n-\m, \delta).
    \end{equation*}
\end{theorem} 

\begin{theorem}[Sample compression bound of \citet{shah_margin-sparsity_2005}]\label{thm:binom_tail}
    For any distribution $\calD$ over $\calX \times \calY$,  for any deterministic reconstruction function $\scriptR$ that outputs sample-compressed predictors $h \in \calH_S$ and for any $\delta \in (0,1]$, with probability at least $1-\delta$ over the draw of $S \sim \calD^n$, we have
    \begin{align*}
        &\forall \bfi \in \scriptP(n): R_{\calD}(\scriptR(S_{\bfi})) \leq \overline{\emph{Bin}}\qty(|\!\bfi^c\!|\widehat{R}_{S_{\bfi^c}}(\scriptR(S_{\bfi})),|\!\bfi^c\!|, \smqty(n \\ \m)^{-1} \tfrac{6}{\pi^2}(\m+1)^{-2}\delta).
    \end{align*}
\end{theorem}

\begin{theorem}[P2L bound of \citet{paccagnan2025pick}] \label{thm:p2l}
    For a fixed compression set size $M$, let $\scriptR(S_{\bfi})$ be the output of P2L which stops when $\m = M$. For any $\delta \in (0,1)$, with probability at least $1-\delta$ over the draw of $S \sim \calD^n$, we have
    \begin{equation*}
        R_{\calD}(\scriptR(S_{\bfi})) \ \leq \ \overline{\varepsilon}\qty(M + |\bfi^c|\widehat{R}_{S_{\bfi^c}}(\scriptR(S_{\bfi})), \ \delta)\,,
    \end{equation*}
where $\overline{\varepsilon}(n, \delta) = 1$ and for $k=0,1,\ldots, n-1$, $\overline{\varepsilon} (k,\delta)$ is the unique solution to the equation $\Psi_{k,\delta}(\varepsilon) = 1$ in the interval $[\frac{k}{n}, 1]$, with
\begin{align*}
    \Psi_{k,\delta}(\varepsilon) &= \frac{\delta}{n} \hspace{1.5mm} \sum_{m=k}^{n-1} \ \ \hspace{-0.2mm} \frac{\smqty(m \\ k)}{\smqty(n \\ k)}(1-\varepsilon)^{-(n-m)}.
 \end{align*}
\end{theorem}

\begin{theorem}[Model compression bound, adapted from \cite{lotfi2022pac}]\label{thm:lotfi_bin}
For any distribution $\calD$ over $\calX \times \calY$, for any prefix-free code $\scriptC$, for any hypothesis class $\calH_{\scriptC}$ such that any $\hat{h} \in \calH_{\scriptC}$ can be defined using the code $\scriptC$ and for any $\delta \in (0,1]$, with probability at least $1-\delta$ over the draw of $S \sim \calD^n$, we have
\begin{align*}
        \forall \hat{h} \in \calH_{\scriptC}: R_{\calD}(\hat{h}) \leq \overline{\emph{Bin}}\qty(n\widehat{R}_{S}(\hat{h}),n, 2^{-l_{\scriptC}(\hat{h})}2^{- 2\log_2 l_{\scriptC}(\hat{h})}\delta).
    \end{align*}
\end{theorem}

\begin{proof}
    Using the prior $p(h) = \frac{1}{Z}2^{-l_{\scriptC}(\hat{h})}2^{- 2\log_2 l_{\scriptC}(\hat{h})}$ of \cite{lotfi2022pac}, the test set bound of \cite{langford2005tutorial} and the union bound, we get 
    \begin{equation}
        \forall \hat{h} \in \calH_{\scriptC}: R_{\calD}(\hat{h}) \leq \overline{\textrm{Bin}}\qty(\kappa,n, \frac{1}{Z}2^{-l_{\scriptC}(\hat{h})}2^{- 2\log_2 l_{\scriptC}(\hat{h})}\delta).
    \end{equation}
    The binomial tail inversion is decreasing with respect to $\delta$, as a smaller $\delta$ means a looser bound. By definition, $Z \leq 1$, which means that 
    \begin{align*}
        \frac{1}{Z}2^{-l_{\scriptC}(\hat{h})}2^{- 2\log_2 l_{\scriptC}(\hat{h})} &\geq 2^{-l_{\scriptC}(\hat{h})}2^{- 2\log_2 l_{\scriptC}(\hat{h})} \\
        \implies  \overline{\textrm{Bin}}\qty(\kappa,n, \frac{1}{Z}2^{-l_{\scriptC}(\hat{h})}2^{- 2\log_2 l_{\scriptC}(\hat{h})}\delta) &\leq \overline{\textrm{Bin}}\qty(\kappa,n, 2^{-l_{\scriptC}(\hat{h})}2^{- 2\log_2 l_{\scriptC}(\hat{h})}\delta).
    \end{align*}
\end{proof}

\begin{theorem}[Chernoff test-set bound of \cite{ foong2022note}, adapted for random coresets]\label{thm:test_set_chernoff_bound}
For any distribution $\calD$ over $\calX \times \calY$, for any predictor $h \in \calH$, for any loss $\ell : \R^C \times \calY \to [B_{\ell}, T_{\ell}]$, for any random sequence $\bfi \in \scriptP(n)$ and for any $\delta \in (0,1]$, with probability at least $1-\delta$ over the draw of $S \sim \calD^n$, we have
    \begin{equation*}
        \calL_{\calD}(A(S_{\bfi})) \leq B_{\ell} + \lambda_{\ell}\kl^{-1}\qty(\frac{\hatL_{S_{\bfi^c}}(A(S_{\bfi})) - B_{\ell}}{\lambda_{\ell}}, \frac{1}{n-\m}\log\frac{1}{\delta}).
    \end{equation*}
\end{theorem}

\begin{theorem}[Sample-compression bound of \cite{bazinet2024}]\label{thm:bazinet}
For any distribution $\calD$ over $\calX \times \calY$,  for any deterministic reconstruction function $\scriptR$ that outputs sample-compressed predictors $h \in \calH_S$, for any loss $\ell: \R^C \times \calY \to [B_{\ell}, T_{\ell}]$ and for any $\delta \in (0,1]$, with probability at least $1-\delta$ over the draw of $S \sim \calD^n$, we have 
\begin{align*}
&\forall \mathbf{i} \in \scriptP(n): \mathcal{L}_{\calD}(\scriptR(S_{\bfi})) \leq B_{\ell} + \lambda_{\ell}\kl^{-1}\qty(\frac{\hatL_{S_{\bfi^c}}(\scriptR(S_{\bfi})) - B_{\ell}}{\lambda_{\ell}}, \frac{1}{|\!\bfi^c\!|}\qty[\log \mqty(n \\ \m) + \log\qty(\frac{2\sqrt{|\!\bfi^c\!|}}{\tfrac{6}{\pi^2}(\m+1)^{-2}\delta})])\,.
\end{align*}
\end{theorem}

\begin{theorem}[Model compression bound, adapted from \cite{lotfi2023non}]\label{thm:lotfi_kl}
For any distribution $\calD$ over $\calX \times \calY$, for any prefix-free code $\scriptC$, for any hypothesis class $\calH_{\scriptC}$ such that any $\hat{h} \in \calH_{\scriptC}$ can be defined using the code $\scriptC$, for any loss $\ell: \R^C \times \calY \to [B_{\ell}, T_{\ell}]$ and for any $\delta \in (0,1]$, with probability at least $1-\delta$ over the draw of $S \sim \calD^n$, we have
 \begin{align*}
&\forall \hat{h} \in \calH_{\scriptC}: \mathcal{L}_{\calD}(\hat{h}) \leq B_{\ell} + \lambda_{\ell}\kl^{-1}\qty(\frac{\hatL_{S}(\hat{h}) - B_{\ell}}{\lambda_{\ell}}, \frac{1}{n}\qty[ l_{\scriptC}(\hat{h}) \log 2 + 2 \log l_{\scriptC}(\hat{h}) + \log\qty(\frac{2\sqrt{n}}{\delta}) ])\,.
\end{align*}
\end{theorem}

\begin{proof}
    We use the same proof technique as \cite{bazinet2024}, but we replace the sample-compression prior with \mbox{$p(h) = \frac{1}{Z}2^{-l_{\scriptC}(\hat{h})}2^{- 2\log_2 l_{\scriptC}(\hat{h})}$}. Similarly to the proof of Theorem~\ref{thm:lotfi_bin}, we can avoid computing the normalizing constant $Z$ as  $\kl^{-1}(q,\epsilon)$ is increasing in $\epsilon$ and, with $Z \leq 1$, we have $$-\log p(h) = l_{\scriptC}(\hat{h})\log 2 + 2\log l_{\scriptC}(\hat{h}) + \log Z \leq l_{\scriptC}(\hat{h})\log 2 + 2\log l_{\scriptC}(\hat{h}).$$
\end{proof}

\begin{theorem}[PAC-Bayes bound of \citet{perez2021tighter}]\label{thm:pbb}
    For any distribution $\calD$ over $\calX \times \calY$, for any set $\calH$ of predictors $h : \calX \to \calY$, for any loss $\ell:\R^C \times \calY \to [B_{\ell}, T_{\ell}]$, for any dataset-independent prior distribution $P$ on $\calH$, for any $\delta, \delta' \in (0,1]$, with probability at least $1-\delta-\delta'$ over the draw of $S \sim \calD^n$ and a set of $V$ predictors $h_1, \ldots, h_V \sim Q$, where $Q$ is a dataset-dependent posterior distribution over $\calH$, we have 
    \begin{equation*}
        \E_{h \sim Q} \calL_{\calD}(h) \leq B_{\ell} + \lambda_{\ell}\kl^{-1}\qty(\kl^{-1}\qty(\frac{1}{V}\sum_{i=1}^V\frac{\hatL_{S}(h_i) - B_{\ell}}{\lambda_{\ell}}, \frac{1}{V}\log\frac{2}{\delta'} ), \frac{1}{n} \qty[\KL(Q ||P) + \ln \qty(\frac{2\sqrt{n}}{\delta})]) \,.
    \end{equation*}
\end{theorem}

\section{Proofs of the Main Results}\label{app:proofs}

\binthm*

\begin{proof}

Given two predictors $h, f \in \calH$ such that $h(\bx) = (h(\bx)_1, \ldots, h(\bx)_C)$ and $f(\bx) = (f(\bx)_1, \ldots, f(\bx)_C)$. We use $f_{\max}(\bx) = \argmax_j f(\bx)_j$ and $h_{\max}(\bx) = \argmax_k h(\bx)_k$ to simplify the notation. Then, the result of \citet{yang2024mind} states that
\begin{equation}\label{eq:yang}
    \left| R_{\calD}(f) -  R_{\calD}(h)\right| \leq \E_{(\bx,y) \sim \calD}\indicator\qty[f_{\max}(\bx) \neq h_{\max}(\bx)].
\end{equation}

For completeness, we restate the proof of this result before continuing on to prove our result. 
    \begin{align*}
    \left| R_{\calD}(f) - R_{\calD}(h)\right| 
    &=\left|\E_{(\bx,y) \sim \calD} \indicator\qty[f_{\max}(\bx) \neq y] - \E_{(\bx,y) \sim \calD} \indicator\qty[h_{\max}(\bx) \neq y]\right|  \\
    &=\left|\E_{(\bx,y) \sim \calD}\left[\indicator\qty[f_{\max}(\bx) \neq y] -  \indicator\qty[h_{\max}(\bx) \neq y]\right]\right|  \\
    &\leq \E_{(\bx,y) \sim \calD}\left|\indicator\qty[f_{\max}(\bx) \neq y] -  \indicator\qty[h_{\max}(\bx) \neq y]\right|  \\
    &\leq \E_{(\bx,y) \sim \calD}\indicator\qty[f_{\max}(\bx) \neq h_{\max}(\bx)]\,.
\end{align*}
        The first inequality is the triangle inequality. The second inequality can easily be shown with a truth table (see Table~\ref{tab:truth_value}).
\begin{table}[h]
    \centering
    \begin{tabular}{c|c|c|c}
         $\indicator\qty[f_{\max}(\bx) \neq y]$ & $\indicator\qty[h_{\max}(\bx)  \neq y]$ & $\mid\!\indicator\qty[f_{\max}(\bx) \neq y] -  \indicator\qty[h_{\max}(\bx)  \neq y]\!\mid$ & $\indicator\qty[f_{\max}(\bx)  \neq h_{\max}(\bx) ]$ \\
         0 & 0 & 0 &  0 \\
         1 & 0 & 1 &  1 \\
         0 & 1 & 1 &  1 \\
         1 & 1 & 0 &  0 or 1 
    \end{tabular}
    \caption{Truth table (0 for false, 1 for true). Let $\calY = \{1, \ldots, C\}$ the set of classes. First line: if both $f_{\max}(\bx)=y$ and $h_{\max}(\bx)=y$, then $f_{\max}(\bx)=h_{\max}(\bx)$. Second line (and similarly for the third line): if $f_{\max}(\bx) \neq y$ and $h_{\max}(\bx) = y$, then $f_{\max}(\bx)\neq h_{\max}(\bx)$. Without loss of generality, suppose $y = 1$. Then, for the last line, if $f_{\max}(\bx) \neq y$ and $h_{\max}(\bx)\neq y$, then we have $f_{\max}(\bx) \in \{2, \ldots, C\}$ and $h_{\max}(\bx) \in \{2, \ldots, C\}$. It is then both possible to have $f_{\max}(\bx) = h_{\max}(\bx) \implies \indicator[f_{\max}(\bx) \neq h_{\max}(\bx)] = 0$ and  $f_{\max}(\bx) \neq h_{\max}(\bx) \implies \indicator[f_{\max}(\bx) \neq h_{\max}(\bx)] = 1$.}
    \label{tab:truth_value}
\end{table}

Equation~\ref{eq:yang} is not computable because it requires knowledge of the true data distribution $\calD$. We now extend this result and upperbound the disagreement. To do so, we start by applying the binomial test-set bound of Langford (2005) (see \cref{thm:langford2005}) on the right-hand side of \cref{eq:yang}, with $p = \E_{(\bx,y) \sim \calD}\indicator\qty[f_{\max}(\bx) \neq h_{\max}(\bx)]$. 

\begin{align*}
    1-\delta \leq&\Prob_{U \sim \calD^m}\qty(\E_{(\bx,y) \sim \calD}\indicator\qty[f_{\max}(\bx) \neq h_{\max}(\bx)] \leq \overline{\mathrm{Bin}}\qty(\sum_{i=1}^m \indicator\qty[f_{\max}(\bx_i) \neq h_{\max}(\bx_i)], m, \delta)) \\
    \leq& \Prob_{U \sim \calD^m}\qty(\left| R_{\calD}(f) - R_{\calD}(h)\right| \leq \overline{\mathrm{Bin}}\qty(\sum_{i=1}^m \indicator\qty[f_{\max}(\bx_i) \neq h_{\max}(\bx_i)], m, \delta)) \tag{Equation~\ref{eq:yang}} \\
    =& \Prob_{U \sim \calD^m}\qty(\left|R_{\calD}(f) - R_{\calD}(h)\right| \leq \overline{\mathrm{Bin}}\qty(m\dzero_U(f, h), m, \delta))  \\
    \leq& \Prob_{U \sim \calD^m}\qty( R_{\calD}(f) \leq  R_{\calD}(h) +  \overline{\mathrm{Bin}}\qty(m\dzero_U(f, h), m, \delta))
\end{align*}
with
\begin{equation*}
    \dzero_U(f, h) = \frac{1}{m}\sum_{i=1}^{m} \indicator\qty[\argmax_j \ f(\bx_i)_j \neq \argmax_k \ h(\bx_i)_k].
\end{equation*}
\end{proof}

For the following theorem, we make the assumption that the loss natively applies a softmax over the output of the predictors. Otherwise, we simply consider $h'(\cdot) = \boldsigma(h(\cdot))$ in the proof.
\softmaxloss*

\begin{proof}[Proof of \cref{thm:softmaxloss}]

With $\boldsigma(f(\bx)) = \left(\sigma(f(\bx), 1), \ldots, \sigma(f(\bx), C)\right)$, we have
\begin{align*}
 \left| \calL_{\calD}(f) - \E_{h \sim Q} \calL_{\calD}(h)\right|
&= \left|\E_{h \sim Q} \qty[\calL_{\calD}(f) - \calL_{\calD}(h)]\right| \\
&= \left|\E_{h \sim Q} \qty[\E_{(\bx,y) \sim \calD} \ell(f(\bx), y) -  \E_{(\bx,y) \sim \calD} \ell(h(\bx),y)]\right| \\
&= \left|\E_{h \sim Q}\E_{(\bx,y) \sim \calD} \qty[ \ell(f(\bx), y) -\ell(h(\bx),y)]\right| \\
&\leq \E_{h \sim Q}\E_{(\bx,y) \sim \calD} \left| \ell(f(\bx), y) -\ell(h(\bx),y)\right|\numberthis \label{eq:lipschitz_CE} \\
&\leq \E_{h \sim Q}\E_{(\bx,y) \sim \calD} K_{\ell}\|\boldsigma(f(\bx)) - \boldsigma(h(\bx))\|_1 \\
&= K_{\ell} \E_{(\bx,y) \sim \calD}\E_{h \sim Q}\|\boldsigma(f(\bx)) - \boldsigma(h(\bx))\|_1.
\end{align*}
The first inequality stems from the triangle inequality. The second inequality comes from the $K_{\ell}$-Lipschitz continuity of $\ell$. 

The 1-norm of the difference of two softmax distributions is always positive and always bounded by $2$, which can easily be proved using the triangle inequality.

\begin{equation*}
    \|\boldsigma(f(\bx)) - \boldsigma(h(\bx))\|_1 \leq \|\boldsigma(f(\bx))\|_1 + \|\boldsigma(h(\bx))\|_1 \leq 2.
\end{equation*}

From Chernoff's bound for random variables in the unit interval \citep{foong2022note}, with ${p = \E_{(\bx,y) \sim \calD}\E_{h \sim Q} \frac{1}{2} \|\boldsigma(f(\bx)) - \boldsigma(h(\bx))\|_1}$, we have 
\begin{equation*}
    \Prob_{U \sim \calD^m}\qty(\E_{(\bx,y) \sim \calD} \E_{h \sim Q} \frac{1}{2} \|\boldsigma(f(\bx)) - \boldsigma(h(\bx))\|_1\leq \kl^{-1}\qty(\E_{h \sim Q} \frac{1}{m}\sum_{i=1}^m \frac{1}{2} \|\boldsigma(f(\bx_i)) - \boldsigma(h(\bx_i))\|_1, \frac{1}{m}\log\frac{1}{\delta})) \geq 1-\delta.
\end{equation*}

Let the difference between the losses be denoted by
\begin{equation*}
    d_U^{K_{\ell}}(f,h) = \|\boldsigma(f(\bx_i)) - \boldsigma(h(\bx_i))\|_1.
\end{equation*}

We finish the proof.
\begin{align*}
    1-\delta &\leq \Prob_{U \sim \calD^m}\qty(\E_{(\bx,y) \sim \calD} \E_{h \sim Q} \frac{1}{2} \|\boldsigma(f(\bx)) - \boldsigma(h(\bx))\|_1 \leq \kl^{-1}\qty(\frac{1}{2} \E_{h \sim Q} d_U^{K_{\ell}}(f,h), \frac{1}{m}\log\frac{1}{\delta})) \\
    &= \Prob_{U \sim \calD^m}\qty(\E_{(\bx,y) \sim \calD}\E_{h \sim Q}  \|\boldsigma(f(\bx)) - \boldsigma(h(\bx))\|_1 \leq 2\kl^{-1}\qty( \frac{1}{2} \E_{h \sim Q} d_U^{K_{\ell}}(f,h), \frac{1}{m}\log\frac{1}{\delta})) \\
    &= \Prob_{U \sim \calD^m}\qty(\E_{h \sim Q}\E_{(\bx,y) \sim \calD} K_{\ell}\|\boldsigma(f(\bx)) - \boldsigma(h(\bx))\|_1 \leq 2K_{\ell}\kl^{-1}\qty( \frac{1}{2}\E_{h \sim Q} d_U^{K_{\ell}}(f,h), \frac{1}{m}\log\frac{1}{\delta})) \\
    &\leq \Prob_{U \sim \calD^m}\qty(\E_{h \sim Q}\E_{(\bx,y) \sim \calD} \left| \ell(f(\bx), y) -\ell(h(\bx),y)\right| \leq  2 K_{\ell} \kl^{-1}\qty( \frac{1}{2} \E_{h \sim Q} d_U^{K_{\ell}}(f,h), \frac{1}{m}\log\frac{1}{\delta})) \\
    &\leq \Prob_{U \sim \calD^m}\qty( \left| \calL_{\calD}(f) - \E_{h \sim Q} \calL_{\calD}(h) \right| \leq 2K_{\ell}\kl^{-1}\qty( \frac{1}{2}\E_{h \sim Q} d_U^{K_{\ell}}(f,h), \frac{1}{m}\log\frac{1}{\delta})) \\
    &= \Prob_{U \sim \calD^m}\qty(\calL_{\calD}(f) \leq \E_{h \sim Q} \calL_{\calD}(h) + 2K_{\ell}\kl^{-1}\qty( \frac{1}{2} \E_{h \sim Q}  d_U^{K_{\ell}}(f,h), \frac{1}{m}\log\frac{1}{\delta})).
\end{align*}
\end{proof}

\fullloss*

\begin{proof}[Proof of \cref{thm:fullloss}]
\begin{align*}
 \left| \calL_{\calD}(f) - \E_{h \sim Q} \calL_{\calD}(h)\right|
&= \left|\E_{h \sim Q} \qty[\calL_{\calD}(f) - \calL_{\calD}(h)]\right| \\
&= \left|\E_{h \sim Q} \qty[\E_{(\bx,y) \sim \calD} \ell(f(\bx), y) -  \E_{(\bx,y) \sim \calD} \ell(h(\bx),y)]\right| \\
&= \left|\E_{h \sim Q}\E_{(\bx,y) \sim \calD} \qty[ \ell(f(\bx), y) -\ell(h(\bx),y)]\right| \\
&\leq \E_{h \sim Q}\E_{(\bx,y) \sim \calD} \left| \ell(f(\bx), y) -\ell(h(\bx),y)\right| \\
&= \E_{(\bx,y) \sim \calD} \E_{h \sim Q} \left| \ell(f(\bx), y) -\ell(h(\bx),y)\right|
\end{align*}
The inequality stems from the triangle inequality. 

To use Chernoff's bound, we need a loss in the unit interval, so we normalize it.
\begin{equation*}
    0 \leq \frac{\left| \ell(f(\bx), y) -\ell(h(\bx),y)\right|}{\max_{y,y'} \ell(y',y) - \min_{y,y'} \ell(y',y)} = \frac{\left| \ell(f(\bx), y) -\ell(h(\bx),y)\right|}{T_{\ell}-B_{\ell}} \leq 1.
\end{equation*}

We set $\lambda_{\ell} \coloneqq T_{\ell}-B_{\ell}$. Then, from Chernoff's bound for random variables in the unit interval \citep{foong2022note}, with $p=\E_{(\bx,y) \sim \calD}\E_{h \sim Q}\frac{1}{\lambda_{\ell}}\left| \ell(f(\bx), y) -\ell(h(\bx),y)\right|$, we have 
\begin{equation*}
    \Prob_{L \sim \calD^m}\qty(\E_{(\bx,y) \sim \calD}  \E_{h \sim Q}\frac{\left| \ell(f(\bx), y) -\ell(h(\bx),y)\right|}{\lambda_{\ell}} \leq \kl^{-1}\qty( \E_{h \sim Q}\frac{1}{m}\sum_{i=1}^m \frac{\left| \ell(f(\bx_i), y_i) -\ell(h(\bx_i),y_i)\right|}{\lambda_{\ell}}, \frac{1}{m}\log\frac{1}{\delta})) \geq 1-\delta.
\end{equation*}

Let the difference between the losses be denoted by
\begin{equation*}
    d_L(f,h) = \frac{1}{m}\sum_{i=1}^m \left| \ell(f(\bx_i), y_i) -\ell(h(\bx_i),y_i)\right|.
\end{equation*}

We finish the proof.
\begin{align*}
    1-\delta &\leq \Prob_{L \sim \calD^m}\qty(\E_{(\bx,y) \sim \calD}  \E_{h \sim Q} \frac{\left| \ell(f(\bx), y) -\ell(h(\bx),y)\right|}{\lambda_{\ell}} \leq \kl^{-1}\qty(  \E_{h \sim Q}\frac{1}{\lambda_{\ell}}d_L(f,h), \frac{1}{m}\log\frac{1}{\delta})) \\
    &= \Prob_{L \sim \calD^m}\qty(\E_{(\bx,y) \sim \calD}  \E_{h \sim Q} \left| \ell(f(\bx), y) -\ell(h(\bx),y)\right| \leq \lambda_{\ell}\kl^{-1}\qty(  \E_{h \sim Q}\frac{1}{\lambda_{\ell}}d_L(f,h), \frac{1}{m}\log\frac{1}{\delta})) \\
    &\leq \Prob_{L \sim \calD^m}\qty( \left| \calL_{\calD}(f) - \E_{h \sim Q} \calL_{\calD}(h) \right| \leq \lambda_{\ell}\kl^{-1}\qty(  \E_{h \sim Q}\frac{1}{\lambda_{\ell}}d_L(f,h), \frac{1}{m}\log\frac{1}{\delta})) \\
    &= \Prob_{L \sim \calD^m}\qty(\calL_{\calD}(f) \leq \E_{h \sim Q} \calL_{\calD}(h) + \lambda_{\ell}\kl^{-1}\qty(  \E_{h \sim Q}\frac{1}{\lambda_{\ell}}d_L(f,h), \frac{1}{m}\log\frac{1}{\delta})).
\end{align*}
\end{proof}

\begin{restatable}{corollary}{corrlip}\label{corr:lip}
    In the setting of Theorem~\ref{thm:fullloss}, for a Lipschitz continuous loss function $\ell$ with constant $K_{\ell}$, with probability at least $1-\delta$ over the sampling of $L \sim \calD^m$, we have :
    \begin{align*}
    \calL_{\calD}(f) &\leq \E_{h \sim Q} \calL_{\calD}(h) + \lambda_{\ell}\kl^{-1}\qty(  \E_{h \sim Q}\frac{1}{\lambda_{\ell}} d_L(f,h), \frac{1}{m}\log\frac{1}{\delta}) \\
    &\leq \E_{h \sim Q} \calL_{\calD}(h) + \lambda_{\ell}\kl^{-1}\qty(  \E_{h \sim Q}\frac{K_{\ell}}{\lambda_{\ell}} \widehat{d_L}(f,h), \frac{1}{m}\log\frac{1}{\delta})
\end{align*}
with 
\begin{equation*}
    \widehat{d_L}(f,h) = \frac{1}{m}\sum_{i=1}^m \|\sigma(f(\bx_i), y_i) - \sigma(h(\bx_i), y_i)\|_1.
\end{equation*}
\end{restatable}

\begin{proof}
Following from the Lipschitz-continuity of $\ell$, we have
    \begin{equation*}
        d_L(f,h) = \frac{1}{m}\sum_{i=1}^m \left| \ell(f(\bx_i), y_i) -\ell(h(\bx_i),y_i)\right| \leq \frac{K_{\ell}}{m}\sum_{i=1}^m \|\sigma(f(\bx_i), y_i) - \sigma(h(\bx_i), y_i)\|_1.
    \end{equation*}
    Moreover, as $\kl^{-1}(q,\epsilon)$ is monotonically increasing in $q$, we have :
    \begin{equation*}
        \kl^{-1}\qty(  \E_{h \sim Q}\frac{1}{\lambda_{\ell}}d_L(f,h), \frac{1}{m}\log\frac{1}{\delta}) \leq \kl^{-1}\qty(  \E_{h \sim Q}\frac{K_{\ell}}{\lambda_{\ell}}\frac{1}{m}\sum_{i=1}^m \|\sigma(f(\bx_i), y_i) - \sigma(h(\bx_i), y_i)\|_1, \frac{1}{m}\log\frac{1}{\delta}).
    \end{equation*}
\end{proof}

\section{Additional Results}

\subsection{Special Cases of the Main Results}\label{ss:lipschitz}
\begin{corollary}
In the setting of Theorem~\ref{thm:fullloss}, with $\alpha\in(0,1)$, $\beta = \frac{C}{\alpha}$, the clamped softmax (see Definition~\ref{def:clamped_softmax})
and the clamped cross-entropy loss $\ell(g(\bx),y) = -\ln(\sigma_{\max}(g(\bx),y))$, with probability at least $1-\delta$ over the sampling of $L \sim \calD^m$, we have
\begin{align*}
    \calL_{\calD}(f) &\leq \E_{h \sim Q} \calL_{\calD}(h) + \ln\left(\beta\right)\kl^{-1}\qty( \E_{h \sim Q} \frac{1}{\ln\left(\beta\right)}d_L(f,h), \frac{1}{m}\log\frac{1}{\delta}) \\
    &\leq \E_{h \sim Q} \calL_{\calD}(h) + \ln\left(\beta\right)\kl^{-1}\qty(  \E_{h \sim Q} \frac{\beta}{\ln\left(\beta\right)}\widehat{d_L}(f,h), \frac{1}{m}\log\frac{1}{\delta})
\end{align*}
with 
\begin{equation*}
    \widehat{d_L}(f,h) = \frac{1}{m}\sum_{i=1}^m \| \sigma_{\max}(f(\bx_i),y_i) - \sigma_{\max}(h(\bx_i),y_i)\|_1.
\end{equation*}
\end{corollary}
\begin{proof}
This corollary is an application of Theorem~\ref{thm:fullloss} and Corollary~\ref{corr:lip} to the clamped cross-entropy loss. In Corollary~\ref{corr:lip}, we want to highlight the difference between the clamped softmax of $f$ and $h$. Thus, we consider the loss $\ell$ as simply $-\ln(u)$, with $u \in (\frac{\alpha}{C},1)$.

We now find the Lipschitz constant $K_{\ell}$ by computing the gradient of $\ell$.
\begin{equation*}
    K_{\ell} = \sup_{u \in (\frac{\alpha}{C},1)} \left| \dv{(-\ln(u))}{u}\right| = \sup_{u \in (\frac{\alpha}{C},1)} \left| \dv{\ln(u)}{u}\right| = \sup_{u \in (\frac{\alpha}{C},1)}\left|\frac{1}{u}\right| = \frac{C}{\alpha}.
\end{equation*}
To simplify the notation in the theorem, we choose $\beta = \frac{C}{\alpha}$.
\end{proof}

\begin{corollary}
In the setting of Theorem~\ref{thm:fullloss}, with $\alpha\in(0,1)$, $\beta = \frac{C}{\alpha}$, the smoothed softmax (see Definition~\ref{def:smooth_softmax})
and the smoothed cross-entropy loss $\ell(g(\bx),y) = -\ln(\sigma_{\emph{smooth}}(g(\bx),y))$, with probability at least $1-\delta$ over the sampling of $L \sim \calD^m$, we have
\begin{align*}
     \calL_{\calD}(f) &\leq \E_{h \sim Q} \calL_{\calD}(h) + \ln \left(1 + (1-\alpha)\beta\right)\kl^{-1}\qty(  \E_{h \sim Q}\frac{1}{\ln \left(1 + (1-\alpha)\beta\right)}d_L(f,h), \frac{1}{m}\log\frac{1}{\delta}) \\
    &\leq \E_{h \sim Q} \calL_{\calD}(h) + \ln \left(1 + (1-\alpha)\beta\right)\kl^{-1}\qty( \E_{h \sim Q} \frac{\beta}{\ln \left(1 + (1-\alpha)\beta\right)} \widehat{d_L}(f,h), \frac{1}{m}\log\frac{1}{\delta})
\end{align*}
with 
\begin{equation*}
    \widehat{d_L}(f,h) = \frac{1}{m}\sum_{i=1}^m \| \sigma_{\emph{smooth}}(f(\bx_i),y_i) - \sigma_{\emph{smooth}}(h(\bx_i),y_i)\|_1.
\end{equation*}
\end{corollary}
\begin{proof}
This corollary is an application of Theorem~\ref{thm:fullloss} and Corollary~\ref{corr:lip} to the clamped cross-entropy loss. In Corollary~\ref{corr:lip},
we want to highlight the difference between the smoothed softmax of $f$ and $h$. Thus, we consider the loss $\ell$ as simply $-\ln(u)$, with $u \in \left(\frac{\alpha}{C},1-\alpha+\frac{\alpha}{C}\right)$.

We now find the Lipschitz constant $K_{\ell}$ by computing the gradient of $\ell$.
\begin{equation*}
    K_{\ell} = \sup_{u \in \left(\frac{\alpha}{C},1-\alpha+\frac{\alpha}{C}\right)} \left| \dv{(-\ln(u))}{u}\right| = \sup_{u \in \left(\frac{\alpha}{C},1-\alpha+\frac{\alpha}{C}\right)} \left| \dv{\ln(u)}{u}\right| = \sup_{u \in \left(\frac{\alpha}{C},1-\alpha+\frac{\alpha}{C}\right)}\left|\frac{1}{u}\right| = \frac{C}{\alpha}.
\end{equation*}
To simplify the notation in the theorem, we choose $\beta = \frac{C}{\alpha}$.
\end{proof}

\subsection{Computable PAC-Bayesian Disagreement Bounds}\label{app:ss:pb_disag}

To compute the bounds exactly, we need to use the same trick as in \cref{thm:pbb}, that is we approximate the sampling of the hypothesis from the posterior $Q$ via Monte Carlo Sampling. This result can be applied to the zero-one loss, as we have $\lambda_{\ell} = 1$ and $d_L(\cdot, \cdot) \leq \dzero_U(\cdot, \cdot)$ by \cref{eq:yang}, which doesn't require a labeled set $L$.
 
\begin{lemma}\label{lemma:disagreement_pbb}
 For any distribution $\calD$ over $\calX \times \calY$, for any predictor $f \in \calH$, for any loss $\ell:\R^C \times \calY \to [0,1]$, for any $\delta \in (0,1]$, with probability at least $1-2\delta$ over the draw of $L \sim \calD^m$ and a set of $V$ predictors $h_1, \ldots, h_V \sim Q$, where $Q$ is a dataset-dependent posterior distribution over $\calH$, we have 
\begin{equation*}
     \calL_{\calD}(f) \leq \E_{h \sim Q}\calL_{\calD}(h) + \lambda_{\ell} \kl^{-1}\qty(\kl^{-1}\qty( \frac{1}{V}\sum_{i=1}^V\frac{1}{\lambda_{\ell}}d_L(f,h_i), \frac{1}{V}\log\frac{1}{\delta}), \frac{1}{m}\log\frac{1}{\delta}).
\end{equation*}
\end{lemma}
\begin{proof}
When computing Theorem~\ref{thm:fullloss} for a Gaussian distribution $Q$ over the predictors, we cannot compute exactly the term 
\begin{equation*}
    \E_{h \sim Q} \frac{1}{\lambda_{\ell}}d_L(f,h)
\end{equation*}
as the expectation doesn't have a closed form. However, similarly to Theorem~\ref{thm:pbb}, we can upper-bound it. Instead of using the two-sided Chernoff bound as in \citet{langford2001not, perez2021tighter}, we use the one-sided Chernoff bound \citep{foong2022note}, as we are explicitly interested in an upper-bound.

With $X_i = \frac{1}{\lambda_{\ell}}d_L(f,h_i)$ and $p = \E[X_i] = \E_{h \sim Q} \frac{1}{\lambda_{\ell}}d_L(f,h)$, with probability at least $1-\delta$ over the sampling of $h_1, \ldots, h_V \sim Q$, we have
\begin{equation*}
    \E_{h \sim Q} \frac{1}{\lambda}d_L(f,h) \leq \kl^{-1}\qty( \frac{1}{V}\sum_{i=1}^V\frac{1}{\lambda_{\ell}}d_L(f,h_i), \frac{1}{V}\log\frac{1}{\delta}).
\end{equation*}

Using the fact that the inverse of the $\kl$ is increasing in its first argument, we have 
\begin{equation}\label{eq:double_kl_pb}
    \kl^{-1}\qty( \E_{h \sim Q}\frac{1}{\lambda_{\ell}}d_L(f,h), \frac{1}{m}\log\frac{1}{\delta}) \leq \kl^{-1}\qty(\kl^{-1}\qty( \frac{1}{V}\sum_{i=1}^V\frac{1}{\lambda_{\ell}}d_L(f,h_i), \frac{1}{V}\log\frac{1}{\delta}), \frac{1}{m}\log\frac{1}{\delta}).
\end{equation}
We finish the proof with a union bound with Theorem~\ref{thm:fullloss} and Equation~\ref{eq:double_kl_pb}.
\end{proof}

\begin{theorem}\label{thm:pbb_disag}
 For any distribution $\calD$ over $\calX \times \calY$, for any predictor $f \in \calH$, for any set $\calH$ of predictors $h : \calX \to \calY$, for any loss $\ell:\R^C \times \calY \to [0,1]$, for any dataset-independent prior distribution $P$ on $\calH$, for any $\delta \in (0,1]$, with probability at least $1-4\delta$ over the draw of $S \sim \calD^n$, $ L \sim\calD^m$ and a set of $2V$ predictors $h_1, \ldots, h_{2V} \sim Q$, where $Q$ is a dataset-dependent posterior distribution over $\calH$, we have 
    \begin{align*}
         \calL_{\calD}(f) &\leq  B_{\ell} + \lambda_{\ell}\kl^{-1}\qty(\kl^{-1}\qty(\frac{1}{V}\sum_{i=1}^V\frac{\hatL_{S}(h_i) - B_{\ell}}{\lambda_{\ell}}, \frac{1}{V}\log\frac{1}{\delta} ), \frac{1}{n} \qty[\KL(\posterior_S ||\prior) + \ln \qty(\frac{2\sqrt{n}}{\delta})]) \\ & + \lambda_{\ell} \kl^{-1}\qty(\kl^{-1}\qty( \frac{1}{V}\sum_{i=V+1}^{2V}\frac{1}{\lambda_{\ell}}d_L(f,h_i), \frac{1}{V}\log\frac{1}{\delta}), \frac{1}{m}\log\frac{1}{\delta}).
    \end{align*}
\end{theorem}

\begin{proof}
    This result simply stems from the union bound of Theorem~\ref{thm:pbb} and Lemma~\ref{lemma:disagreement_pbb}. We remove the factor of $2$ in the $\frac{1}{m}\log\frac{2}{\delta}$ by using the one-sided Chernoff bound, similarly to the proof of Lemma~\ref{lemma:disagreement_pbb}.
\end{proof}

\subsection{Using Other Comparator Functions}

Up until now, all the results are presented with the inverse of the $\kl$. This function gives the tightest bounds, but the bounds obtained with this function are not intuitive. For example, it might be beneficial to use a bound with an analytical upper-bound such as the quadratic loss $\Delta_2(q,p) = 2(q-p)^2$ or Catoni's distance $\Delta_C(q,p) = -\ln (1-p(1-e^{-C})) - Cq$. Both are upper-bounded by $\kl(q,p)$, which can be proven by Pinsker's inequality for the quadratic loss and Proposition 2.1 of \citep{germain2009pac} for Catoni's distance. We present the following result, which is very intuitive and has probably been used before. However, we were not able to find it in the literature, which is why we prove it formally here. 

\begin{lemma}\label{lem:1}
    If $\Delta(q,p) \leq \kl(q,p) \forall q,p\in [0,1]$, for any $\epsilon > 0$, we have $\kl^{-1}(q,\epsilon) \leq \Delta^{-1}(q,\epsilon)$. 
\end{lemma}

\begin{proof}
    For any comparator function, we have 
    \begin{align*}
        \Delta^{-1}(q,\epsilon) &= {\arg\sup}_{0\leq p \leq 1} \left\{\Delta(q,p) \leq \epsilon\right\} \\
        &= {\arg\sup}_{q\leq p \leq 1} \left\{\Delta(q,p) \leq \epsilon\right\}.
    \end{align*}
Indeed, for most $\Delta$, there exists two solutions, one where $p \leq q$ and one where $q \leq p$. Obviously, as we take the arg-supremum, the second solution is always chosen. Thus, we consider only $q \leq p \leq 1$.

Let $p_1 = \kl^{-1}(q,\epsilon)$ be the solution such that $\kl(q,p_1) = \epsilon$. Let $p_2 = \Delta^{-1}(q,\epsilon)$ be the solution such that $\Delta(q,p_2) = \epsilon$. From the assumption that $\Delta(q,p) \leq \kl(q,p) \forall q,p\in [0,1]$, we have
\begin{align*}
    \kl(q,p_1) = \Delta(q,p_2) =\epsilon  \land \Delta(q,p_1) \leq \kl(q,p_1) \implies \Delta(q,p_1) \leq \Delta(q,p_2).
\end{align*}

As $\Delta(q,p)$ is increasing for $p \geq q$, we have that : 
\begin{equation*}
    \Delta(q,p_1) \leq \Delta(q,p_2) \implies p_1 \leq p_2 \implies \kl^{-1}(q,\epsilon) \leq \Delta^{-1}(q,\epsilon).
\end{equation*}
\end{proof}

\section{Experiments} \label{app:experiments}

\paragraph{Devices.}
The experiments were run on three different devices. The target neural networks, the experiments with Pick-To-Learn and the quantization experiments were computed on a computer with Python 3.12.3 and a NVIDIA GeForce RTX 4090. The coreset experiments and the model compression experiments were computed with Python 3.12.4 and a NVIDIA H100 SXM5. Finally, the PAC-Bayesian experiments and the model distillation experiments were computed with Python 3.12.4 and a NVIDIA A100 SXM4.

\paragraph{Libraries.}

All libraries used can be found within the code here :  \url{https://github.com/GRAAL-Research/Bound-to-Disagree}. Notably, we use DeepCore \citep{guo2022deepcore} (MIT license), PACTL \citep{lotfi2022pac} (Apache 2.0 License), PAC-Bayes with Backprop \citep{perez2021tighter} (CC-BY 4.0 License), PyTorch \citep{Ansel_PyTorch_2_Faster_2024} (BSD 3-Clause License), Lightning \citep{Falcon_PyTorch_Lightning_2019} (Apache 2.0 license), Loralib \citep{hu2022lora} (MIT License), Transformers \citep{wolf2020huggingface} (Apache 2.0 License), Scikit-Learn \citep{scikit-learn} (BSD 3-Clause License), NumPy \citep{harris2020numpy} (NumPy license), Weight and Biases \citep{wandb} (MIT License), ScheduleFree \citep{defazio2024road} (Apache 2.0 License), TorchAO \citep{or2025torchao} (BSD 3-Clause License). We also use the KL inversion function of \citep{viallard2021self}, which is distributed under MIT License.

\paragraph{Datasets.}

The datasets used are the MNIST dataset \citep{lecun1998_mnist} (MIT License), the CIFAR10 dataset \citep{cifar10} and the Amazon polarity dataset \citep{zhang2015character} (Apache 2.0 License). There is no explicit license for CIFAR10, but the authors simply ask the user to cite this technical report : \citet{cifar10}.

\subsection{Target Model}

For MNIST, the models were trained using data augmentation for 200 epochs or until the validation error hasn't decreased for 10 epochs. For CIFAR10, the model was trained for 200 epochs. The models on Amazon polarity are trained for 10 epochs or until the validation error hasn't decreased for 2 epochs. The models for MNIST and CIFAR10 are randomly initialized, but for Amazon polarity the models are initialized with the pretrained weights from the Transformers library \citep{wolf2020huggingface}. In some experiments, to remove the need for learning rate scheduling, we try the optimizers SGDFree and AdamFree from the ScheduleFree library \citep{defazio2024road} or the COCOB optimizer from \citet{orabona2017training}. Other times, we use the OneCycle learning rate scheduler \citep{smith2019super} implemented in PyTorch.  

We present the hyperparameter grids used for each problem.

\textbf{Hyperparameters for MNIST}

\begin{multicols}{2}
\begin{itemize}
    \item Dropout probability : $\{0.2, 0.5\}$
    \item Optimizer : $\{$ Adam, AdamFree, COCOB $\}$
    \item Training learning rate : $\{10^{-3}, 10^{-4}\}$
    \item Weight decay :  $\{0.1, 0.01, 0.001\}$
\end{itemize}
\end{multicols}

\textbf{Hyperparameters for CIFAR10} 

\begin{multicols}{2}
\begin{itemize}
    \item Model type : $\{$ CNN, ResNet18 $\}$
    \item Dropout probability : 0.0
    \item Optimizer : SGD
    \item Learning rate scheduler : OneCycle
    \item Training learning rate : $\{0.05, 0.01, 0.005\}$
    \item Weight decay :  $\{10^{-3}, 10^{-4}\}$
\end{itemize}
\end{multicols}

\textbf{Hyperparameters for Amazon polarity} 

\begin{multicols}{2}
\begin{itemize}
    \item Model type : $\{$ DistilBERT, GPT2$\}$
    \item Max epochs : 2
    \item Dropout probability : 0.0
    \item Optimizer : SGD
    \item Learning rate scheduler : OneCycle
    \item Training learning rate : $2\times 10^{-5}$
\end{itemize}
\end{multicols}

\subsection{Partition-Based Bounds}
 We compute the bound both for the bounded cross-entropy loss and the zero-one loss for all of our target models. For each model and for each loss, we do a grid search over the following hyperparameters : 
\begin{multicols}{2}
\begin{itemize}
    \item Number of subsets $K$ : [5, 10, 20, 50, 100, 200]
    \item $\alpha$ = $[20, 50, 100, \frac{n (K + n)}{K(4n-3)}]$
\end{itemize}
\end{multicols}

Given one hyperparameter combination, we choose $\gamma = (\frac{\delta}{2})^{\frac{-1}{\alpha}}$ such that the bound holds in $1-\gamma^{-\alpha} - \frac{\delta}{2} = 1 - ((\frac{\delta}{2})^{\frac{-1}{\alpha}})^{-\alpha} - \frac{\delta}{2} = 1-\delta$ with $\delta = 0.01$. We use a union bound over the 24 combinations of the grid search to get a bound that holds simultaneously for all possibilities. 

\subsubsection{Ablation Study on the Partitioning Method}\label{app:partition}

To compute their bound, \citet{than2025non} partitions the space using a K-means clustering applied to the training set. However, to the best of our knowledge, this choice of data-dependent partition seems to violate the statement of their theorem. To compute the bound, we try two techniques : computing the partition on the unlabeled disagreement set and using random clusters. As their bound doesn't consider unlabeled data points, the disagreement approach is completely valid. However, this is not the setting of the original paper. For the random clusters, we sample $K$ centroids from a uniform distribution and assign the data points to the cluster of the closest centroid. We repeat this experiments five times and consider the best cluster.

\begin{figure}[!h]
    \centering
    \begin{subfigure}[b]{0.49\columnwidth}
            \includegraphics[width=\linewidth]{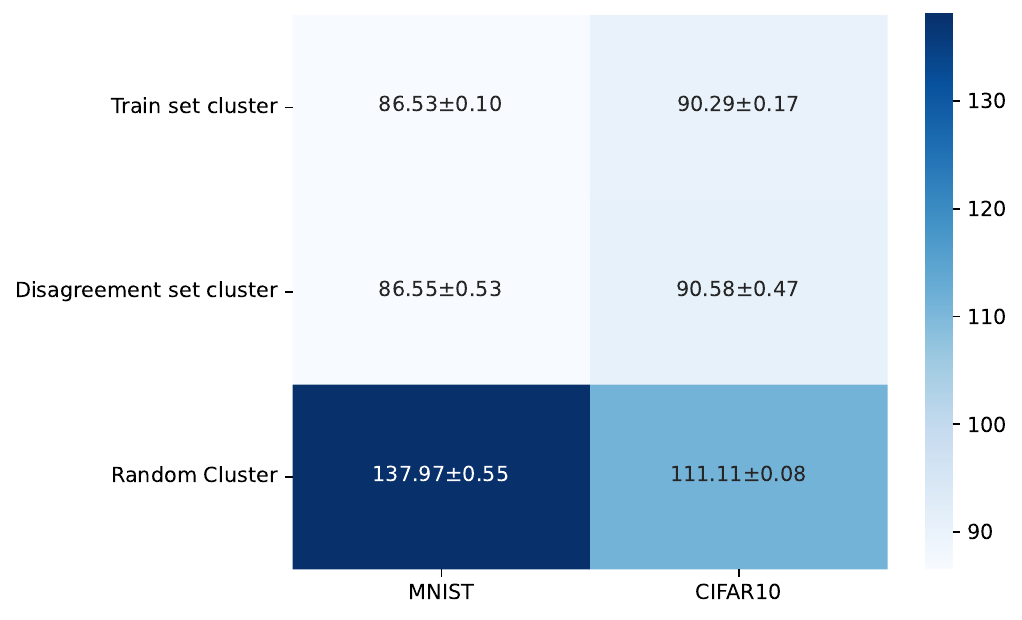}
            \caption{Zero-one loss}
    \end{subfigure}
    \begin{subfigure}[b]{0.49\columnwidth}
            \includegraphics[width=\linewidth]{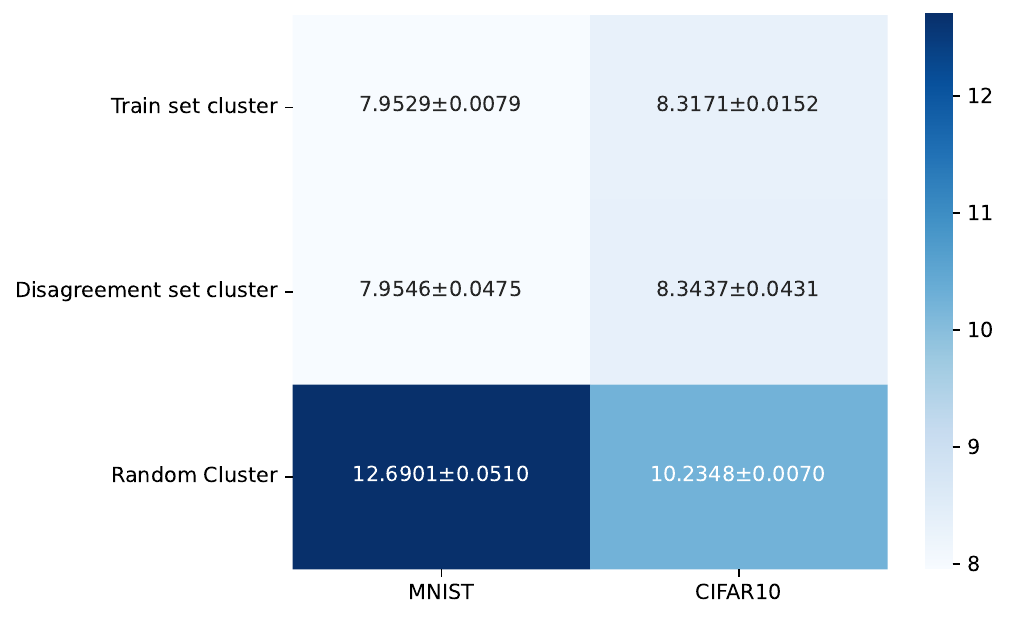}
            \caption{Smoothed cross-entropy loss}
    \end{subfigure}
    \caption{Comparison of the different approaches to compute the partition-based bounds. We consider 5 different seeds for the random clusters. }
    \label{fig:partition_based}
\end{figure}

Although the train set cluster and the disagreement set cluster were obtained using different datasets, they are almost identical. Moreover, we can see that in their setting (without a disagreement set), the only valid option is to use random clusters, which leads to vacuous generalization bounds. Surprisingly enough, the bounds are actually worse on MNIST than CIFAR10.

\subsection{Norm-Based Bounds}

We compute the bound of \cref{thm:norm} for each target model by adapting the code provided by \citep{galanti2023normbased}. As suggested in the paper, we use $\gamma = 1$.

\subsection{Sample Compression Experiments}\label{ss:appendix_sample_compression}

We use the Pick-To-Learn with Early Stopping \citep{marks2025pick} implementation of \citet{bazinet2024}. We use the coreset methods implemented by the DeepCore library \citep{guo2022deepcore}. All approach use \cref{thm:binom_tail} for the zero-one loss and \cref{thm:bazinet} for the cross-entropy loss. The exceptions include Pick-To-Learn with the zero-one loss, as Pick-To-Learn enjoys the tight generalization bounds of \cref{thm:p2l}, for which it was designed. The other exception is the Random Coreset approach, which can be used with the following test set bounds : \cref{thm:test_set_data} for the zero-one loss and \cref{thm:test_set_chernoff_bound} for the cross-entropy loss.

We now prove that coreset methods are sample-compression algorithms. Let $\Ccal : \cup_{1\leq n \leq \infty} (\calX \times \calY)^n \to \cup_{m \leq n} (\calX \times \calY)^m$ denote some coreset method that takes a dataset $S$ and returns a compression set $S_{\bfi}$. Let $\mathrm{SGD}$ denote stochastic gradient descent. Then, we can apply the sample compression bounds for the predictors outputted by $A : \mathrm{SGD} \circ \Ccal$ with the compression set $S_{\bfi} = \Ccal(S)$ and the reconstruction function $\scriptR = \mathrm{SGD}$. Then, we have proven that $A(S) = \scriptR(S_{\bfi}) = \mathrm{SGD}(\Ccal(S))$.

We present the hyperparameter grids. 

\textbf{Hyperparameters for P2L on MNIST}
\begin{multicols}{2}
\begin{itemize}
    \item Dropout probability : $\{0.2, 0.5\}$
    \item Optimizer : $\{$ Adam, AdamFree $\}$
    \item Training learning rate : $\{10^{-3}, 10^{-4}\}$
    \item Weight decay : $\{0.01, 0.1\}$
    \item Smoothing parameter $\alpha$ : $\{10^{-3}, 10^{-4}, 10^{-5}\}$
    \item Softmax : $\{$ Smoothed, Clamped$\}$
\end{itemize}
\end{multicols}

\textbf{Hyperparameters for P2L on CIFAR10}
\begin{multicols}{2}
\begin{itemize}
    \item Model type : $\{$ CNN, ResNet18 $\}$
    \item Dropout probability : 0.5
    \item Optimizer : $\{$ SGD, SGDFree $\}$
    \item Training learning rate : $\{0.05, 0.01, 0.005\}$
    \item Weight decay : $\{10^{-3}, 10^{-4}\}$
    \item Smoothing parameter $\alpha$ : $\{10^{-3}, 10^{-4}, 10^{-5}\}$
    \item Softmax : $\{$ Smoothed, Clamped$\}$
\end{itemize}
\end{multicols}

\textbf{Hyperparameters for coresets on MNIST}
\begin{multicols}{2}
\begin{itemize}
    \item Dropout probability : $\{0.2, 0.5\}$
    \item Coreset method : $\{$ Random Coreset, k-Center Greedy, Forgetting, DeepFool $\}$
    \item Percentage of dataset used as coreset : $\{0.1\%, 0.5\%, 1\%, 5\%, 10\%, 15\%, 20\%\}$
    \item Optimizer : $\{$ Adam, AdamFree $\}$
    \item Training learning rate : $\{10^{-3}, 10^{-4}\}$
    \item Weight decay : $\{0.01, 0.1\}$
    \item Smoothing parameter $\alpha$ : $\{10^{-3}, 10^{-4}, 10^{-5}\}$
    \item Softmax : $\{$ Smoothed, Clamped$\}$
\end{itemize}
\end{multicols}

\textbf{Hyperparameters for coresets on CIFAR10}
\begin{multicols}{2}
\begin{itemize}
    \item Model type : ResNet18
    \item Dropout probability : 0.5
    \item Coreset method : $\{$ Random Coreset, k-Center Greedy, Forgetting, DeepFool $\}$
    \item Percentage of dataset used as coreset~: $\{0.1\%, 0.5\%, 1\%, 5\%, 10\%, 15\%, 20\%, 30\%, 40\%, 50\%\}$
    \item Optimizer : $\{$ SGD, SGDFree $\}$
    \item Training learning rate : $\{0.01, 0.05\}$
    \item Weight decay : $5\times 10^{-4}$
    \item Smoothing parameter $\alpha$ : $\{10^{-3}, 10^{-4} \}$
    \item Softmax : $\{$ Smoothed, Clamped$\}$
\end{itemize}
\end{multicols}

We report the results for the sample compression experiments on CIFAR10 in the following \cref{fig:bar_graph_cifar10}.
\begin{figure*}[!h]
    \centering
    \includegraphics[width=0.49\linewidth]{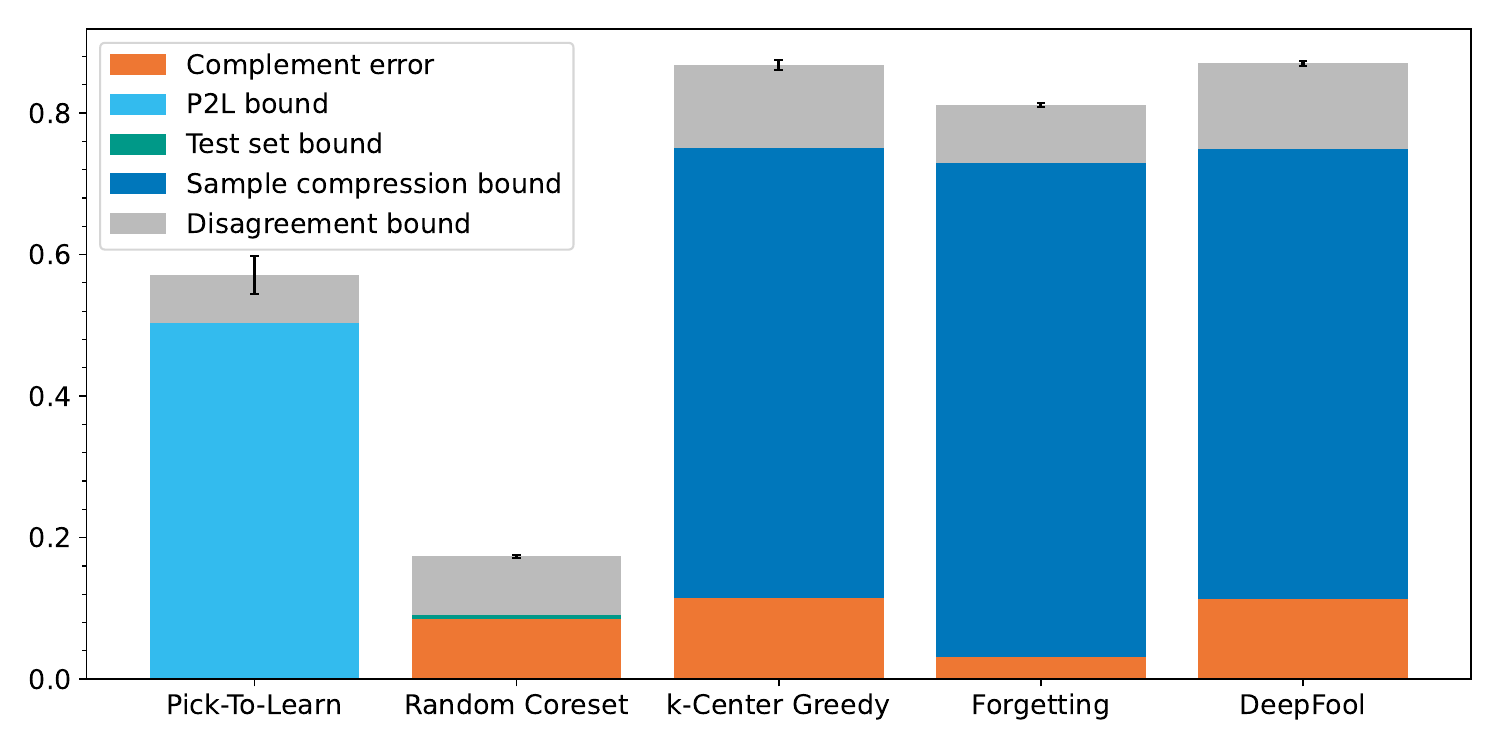}
    \includegraphics[width=0.49\linewidth]{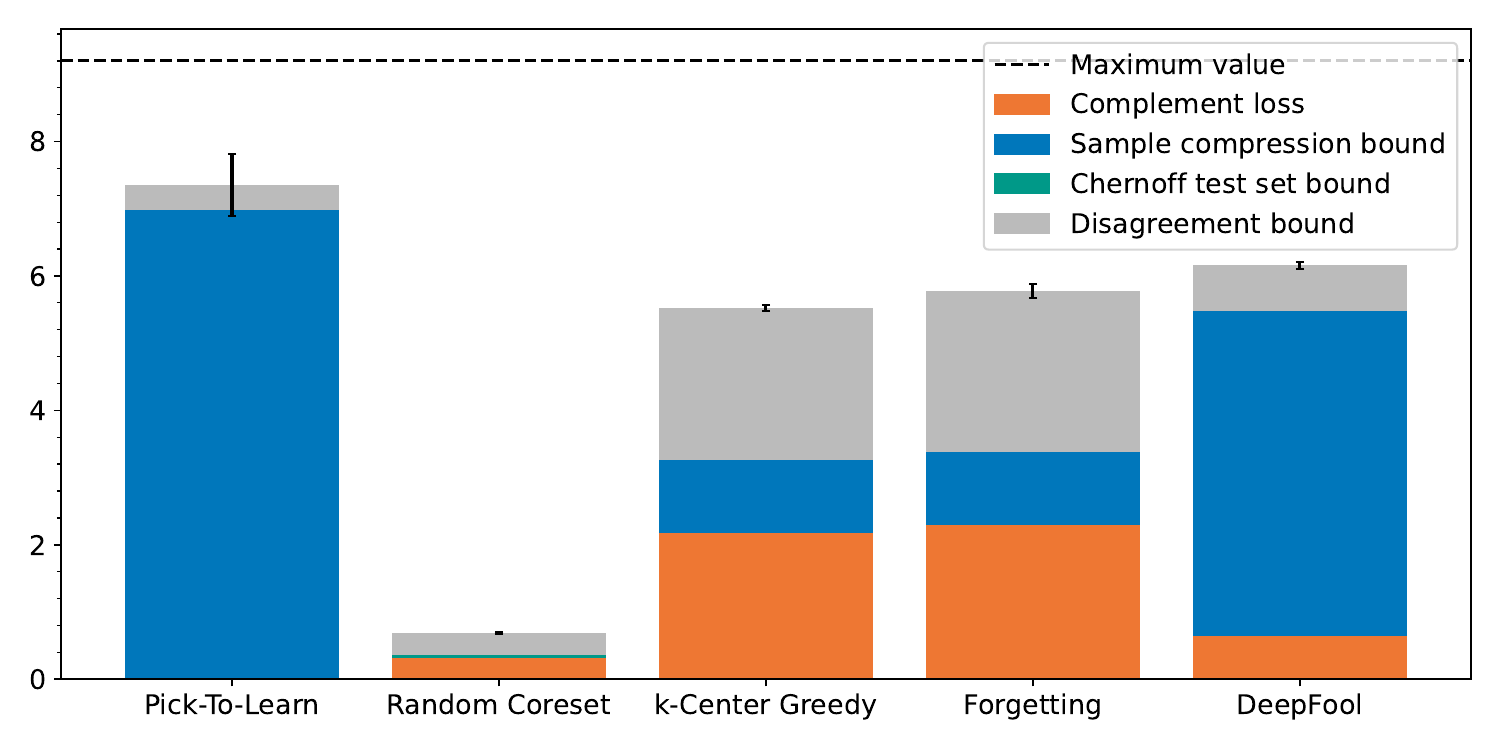}
    \caption{Illustration of the behavior of our disagreement bounds on CIFAR10 using sample-compression methods.}
    \label{fig:bar_graph_cifar10}
\end{figure*}

\subsubsection{Ablation Study for the Bounded Cross-Entropy Loss}\label{sss:ablation}

To choose which hyperparameter to use for the bounded cross-entropy loss, we compare the average cross-entropy loss obtained over all the combinations of hyperparameters for the coreset methods. After computing all possible values, we take the mean over all dimensions except the smoothing parameter $\alpha$ and the softmax type : clamped or smoothed. We report the results in \cref{fig:ablation_softmax}. 

For both experiments, the optimal parameter seems to be $\alpha = 10^{-3}$. The gap between all values of $\alpha$ are not significant, as their standard deviation overlaps. However, on the 448 combinations of hyperparameters for MNIST (with 5 seeds per combinations) and 320 combinations for CIFAR10 (with 5 seeds per combinations), it seems like the parameter $\alpha = 10^{-3}$ leads to the best result. As for the type of softmax, the smoothed softmax seems to lead to slightly better results. Thus, we choose the smoothed softmax with $\alpha = 10^{-3}$ in all other experiments.

\begin{figure}[!h]
    \begin{subfigure}[b]{0.49\columnwidth}
        \centering
            \includegraphics[width=\columnwidth]{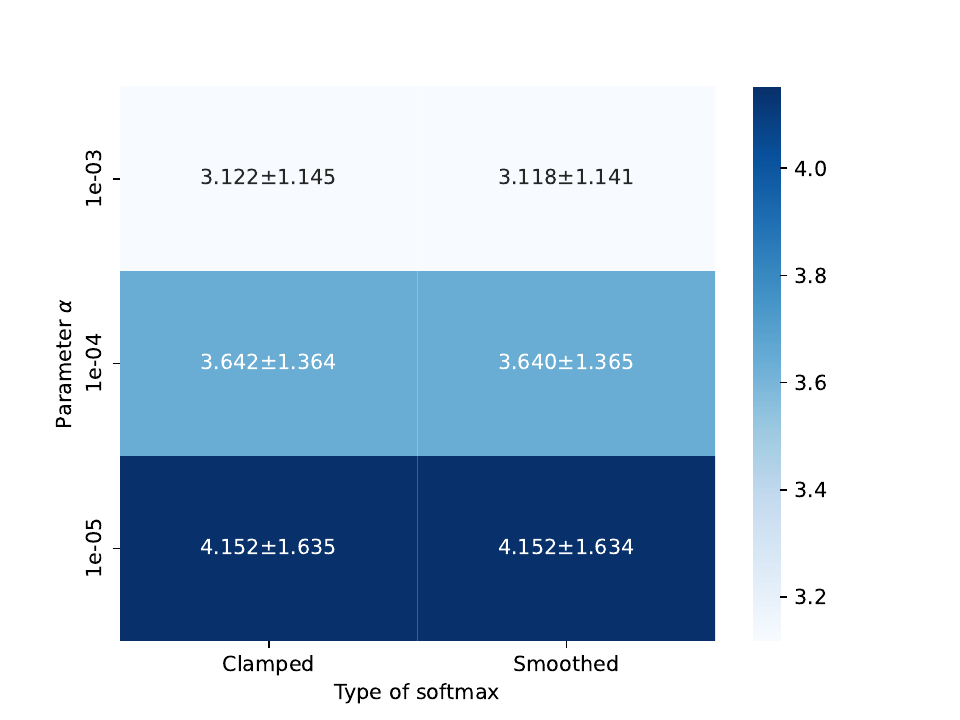}
    \caption{MNIST}
    \label{fig:ablation_mnist}
    \end{subfigure}
    \hfill
    \begin{subfigure}[b]{0.49\columnwidth}
        \centering
            \centering
    \includegraphics[width=\linewidth]{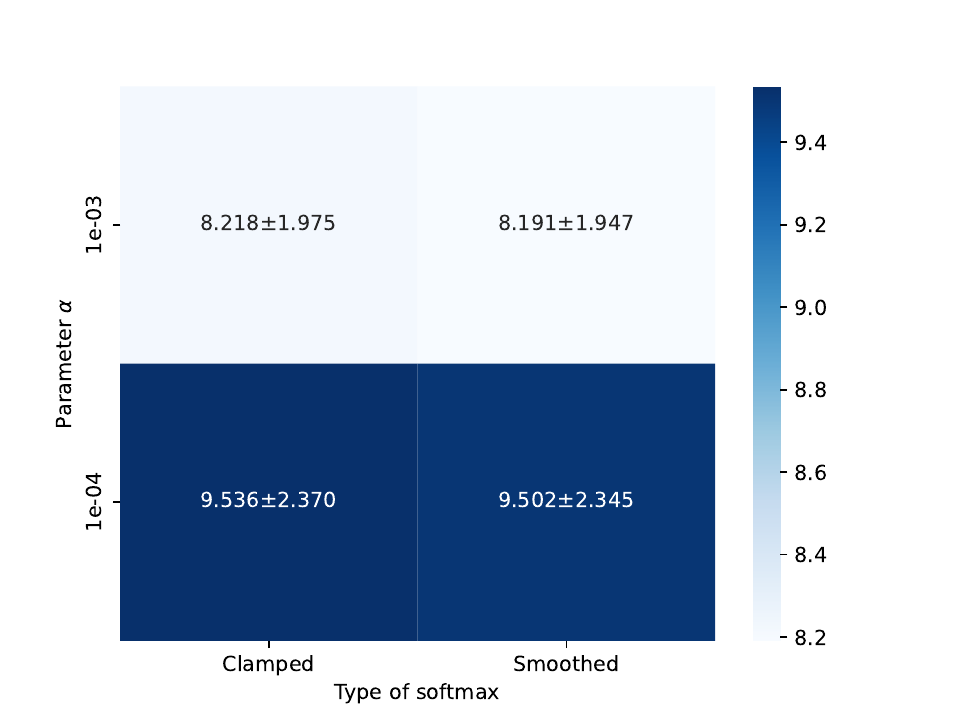}
    \caption{CIFAR10}
    \label{fig:ablation_cifar10}
    \end{subfigure}
    \caption{Ablation study on the hyperparameters of the bounded cross-entropy loss. We report the mean of the cross-entropy loss over all coreset methods experiments. }
    \label{fig:ablation_softmax}
\end{figure}

\subsubsection{Ablation Study for the Size of the Disagreement Set}

Although acquiring unlabeled data is generally easier and cheaper than acquiring labeled data, the cost of generating a large unlabeled set of data can still be prohibitive. Thus, in \cref{fig:ablation_disagreement}, we study the effect of the size of the disagreement set on the disagreement bound of \cref{thm:bin_disag} across 20 subset sizes and five seeds. To do so, we train CNNs on MNIST with random coresets. The original disagreement set contains 10800 data points. The smallest subset contains 1\% of the disagreement set, comprising 108 data points.

We observe that when the disagreement set is smaller than 1000 data points, the bound has high variance. However, when the number of data points becomes greater than a thousand, the improvement in the tightness of the bound becomes marginal. From this small experiment, we argue that the minimum number of data points needed to obtain tight bounds would be between a thousand and two thousand data points.

\begin{figure}[!t]
    \centering
    \includegraphics[width=0.7\linewidth]{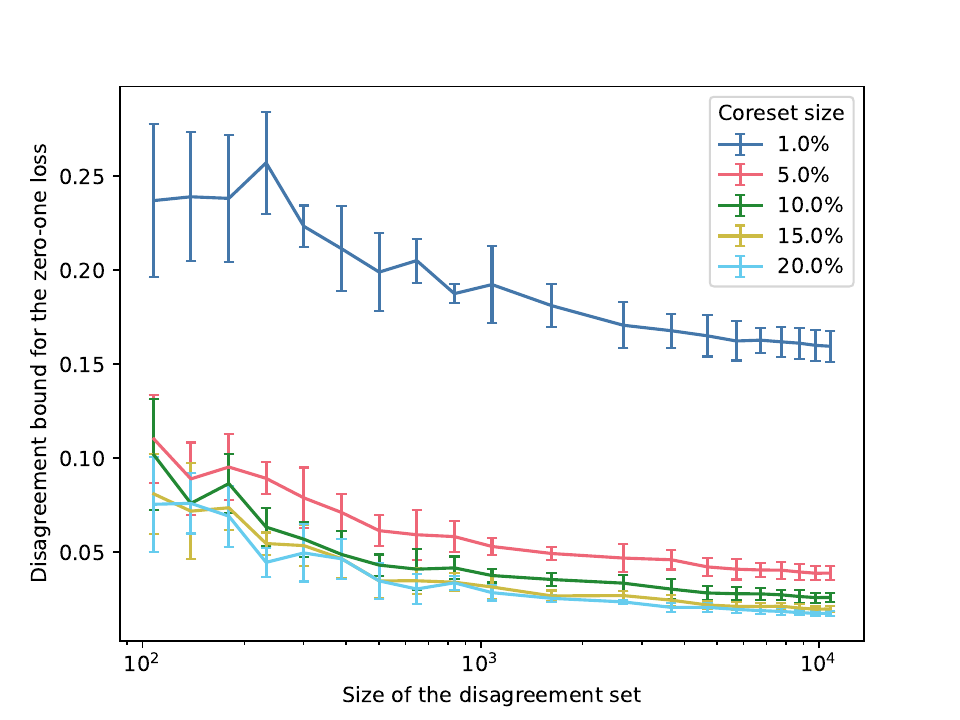}
    \caption{Illustration of the behavior of the disagreement bound with respect to the size of the disagreement set.}
    \label{fig:ablation_disagreement}
\end{figure}

\subsection{Model Compression Experiments}

For the model compression experiments, we pretrain randomly initialized models on a subset of the training set, over which we then add LoRA adapters implemented by \citet{hu2022lora} and the subspace compression approach implemented by \citet{lotfi2022pac}. This SubLoRA model is then trained on the remaining data points, where the bound is also computed.

\textbf{Hyperparameters for MNIST}
\begin{multicols}{2}
\begin{itemize}
    \item Portion of training set for pretraining : 20\%
    \item Number of pretraining epochs : 100
    \item Pretraining learning rate : $10^{-3}$
    \item Dropout probability : 0.0
    \item Optimizer : AdamFree
    \item Training learning rate : $\{10^{-2},  10^{-4}\}$
    \item Weight decay : 0.1
    \item Random subspace projector : Dense
    \item Dimension of the subspace : $\{100, 500, 1000 \}$
    \item Level of quantization : $\{5, 20\}$
    \item Rank of LoRA : $\{0, 2, 4\}$
\end{itemize}
\end{multicols}

\textbf{Hyperparameters for CIFAR10}
\begin{multicols}{2}
\begin{itemize}
    \item Portion of pretraining set : 50\%
    \item Number of pretraining epochs : 150
    \item Pretraining learning rate : 0.05
    \item Dropout probability : 0.2
    \item Optimizer : SGD
    \item Learning rate scheduler : OneCycle
    \item Training learning rate : $\{0.01, 0.005\}$
    \item Max epochs : $\{500, 1000\}$
    \item Weight decay : $10^{-4}$
    \item Random subspace projector : $\{$ Dense, Rounded Double Kronecker $\}$
    \item Dimension of the subspace : $\{100, 500 \}$
    \item Level of quantization : $\{5, 10, 20\}$
    \item Rank of LoRA : $\{0, 2, 4\}$
\end{itemize}
\end{multicols}

We report the results for the cross-entropy loss on both MNIST and CIFAR10 in \cref{tab:model_comp_loss}.
\begin{table*}[!h]
    \centering
    \caption{Generalization bounds on the smoothed cross-entropy loss achieved on MNIST and CIFAR10 using model compression bounds.}
    \begin{tabular}{lccccccc}
    \toprule
        \multirow{2}{4em}{Dataset} & \multicolumn{3}{c}{Surrogate model} & \multicolumn{2}{c}{Target model}  \\ \cmidrule(l{2pt}r{2pt}){2-4} \cmidrule(l{2pt}r{2pt}){5-6}
         & Test loss & Size (KB) & MC Bound & Test loss & Our bound\\ \midrule 
    MNIST & 0.0273$\pm$0.0015 & 0.0356$\pm$0.0008 & 0.1338$\pm$0.0040 &  0.0150$\pm$0.0006& 0.1756$\pm$0.0041 \\
CIFAR10 & 0.7005$\pm$0.0243 & 0.0336$\pm$0.0014 & 1.0978$\pm$0.0195 & 0.2247$\pm$0.0060 & 1.7851$\pm$0.0223 \\
         \bottomrule
    \end{tabular}
    \label{tab:model_comp_loss}
\end{table*}

\subsection{PAC-Bayes Experiments}

For the PAC-Bayes experiments, we pretrain randomly initialized models on a subset of the training set. We then add Gaussian distributions over each weight, with learnable mean and standard deviation parameters. The stochastic model is then trained on the remaining data points, where the bound is also computed.

\begin{table*}[!h]
    \centering
    \caption{Generalization bounds on the smoothed cross-entropy loss achieved on MNIST and CIFAR10 using PAC-Bayes bounds.}
    \begin{tabular}{lccccccc}
    \toprule
        \multirow{2}{4em}{Dataset} & \multicolumn{3}{c}{Surrogate model} &  \multicolumn{2}{c}{Target model}  \\ \cmidrule(l{2pt}r{2pt}){2-4}  \cmidrule(l{2pt}r{2pt}){5-6}
        & Test loss & $\KL$ & PB Bound & Test loss & Our bound\\ \midrule 
MNIST & 0.0286$\pm$0.0030 & 0.7425$\pm$0.1188 & 0.1235$\pm$0.0059 &  0.0150$\pm$0.0006 & 0.2592$\pm$0.0087 \\
CIFAR10 & 0.5440$\pm$0.0568 & 1.6239$\pm$0.9202 & 0.7404$\pm$0.0744 & 0.2247$\pm$0.0060 & 1.4204$\pm$0.1377 \\
         \bottomrule
    \end{tabular}
    \label{tab:pbb_loss}
\end{table*}

\textbf{Hyperparameters for MNIST}
\begin{multicols}{2}
\begin{itemize}
    \item Portion of pretraining set : 20\%
    \item Number of pretraining epochs : 100
    \item Pretraining learning rate : $10^{-3}$
    \item Dropout probability : 0.0
    \item Optimizer : Adam 
    \item Training learning rate : $\{10^{-3}, 5\times 10^{-3},  10^{-4}\}$
    \item Weight decay : $\{0.01, 0.1\}$
    \item Smoothing parameter $\alpha$ : $10^{-3}$
    \item Softmax : $\{$ Smoothed, Clamped$\}$
    \item Standard deviation of prior : $\{10^{-2}, {5\times 10^{-2}}, 10^{-3}, 10^{-4}\}$
    \item Number of Monte Carlo sampling : 2000
\end{itemize}
\end{multicols}

\textbf{Hyperparameters for CIFAR10}
\begin{multicols}{2}
\begin{itemize}
    \item Portion of pretraining set : $\{60\%, 70\%\}$
    \item Number of pretraining epochs : 150
    \item Pretraining learning rate : $0.05$
    \item Dropout probability : 0.2
    \item Optimizer : SGD
    \item Training learning rate : $\{10^{-3}, 5\times 10^{-3},  10^{-4}\}$
    \item Weight decay : $10^{-4}$
    \item Smoothing parameter $\alpha$ : $10^{-3}$
    \item Softmax : $\{$ Smoothed, Clamped$\}$
    \item Standard deviation of prior : $\{10^{-2}, {5\times 10^{-2}}, 10^{-3}, 10^{-4}\}$
    \item Number of Monte Carlo sampling : 5000
\end{itemize}
\end{multicols}

\begin{table}[!t]
    \centering
    \caption{Generalization bounds on the zero-one loss achieved via model distillation on Amazon polarity using model compression (MC) bounds. All metrics presented are in percents (\%), except the size.}
    \begin{tabular}{lcccccccc}
    \toprule
        \multirow{2}{4em}{Model} & \multicolumn{3}{c}{Surrogate model} & \multicolumn{2}{c}{Target model}  \\ \cmidrule(l{2pt}r{2pt}){2-4} \cmidrule(l{2pt}r{2pt}){5-6}
        & Test error  & Size (KB) & MC Bound & Test error & Our bound\\ \midrule 
        DistilBERT & 7.61$\pm$0.07 & 2.02$\pm$0.01 &  14.37$\pm$0.08 & 6.19$\pm$0.15 & 21.35$\pm$0.32 \\
        GPT2 & 16.37$\pm$0.94 & 2.04$\pm$0.43 &  25.05$\pm$0.91 & 5.51$\pm$0.11 & 43.34$\pm$1.52 \\
         \bottomrule
    \end{tabular}
    \label{tab:amazon_distillation}
\end{table}

\subsection{Model Distillation on Amazon Polarity}

We start with DistilBERT and GPT2 with the pretrained weights of the Transformers library. We freeze the model and apply a SubLoRA model on the top half of the model. For the adaptative quantization, we use the implementation of \citet{lotfi2022pac}, a learning rate of $10^{-2}$ for DistilBERT and a learning rate of $10^{-4}$ for GPT2. We produce pseudo-labels for the unlabeled using the target model and train the model to classify correctly both the labeled set $S$ and the unlabeled set $U$ with the pseudo-labels. 

\begin{multicols}{2}
\begin{itemize}
    \item Max epochs : 5
    \item Dropout probability : 0.0
    \item Optimizer : Adam
    \item Training learning rate : $10^{-6}$
    \item Weight decay : $0.01$
    \item Dimension of the subspace : $\{5000, 10000\}$
    \item Level of quantization : 250
    \item Rank of LoRA : 4
\end{itemize}
\end{multicols}

\subsection{Quantization Experiments on Amazon Polarity}

We use the target models and quantize them following this hyperparameter grid. We do not consider the quantization-aware training with 2 bits, as both TorchAO and AO did not consider this setting.  
\begin{multicols}{2}
\begin{itemize}
    \item Max epochs : 2
    \item Dropout probability : 0.0
    \item Optimizer : Adam
    \item Training learning rate : $10^{-6}$
    \item Weight decay : $0.01$
    \item Number of bits : $\{2, 4, 8\}$
    \item Quantization-aware training : $\{$ Yes, No $\}$
    \item Delta for the Huber loss : 0.2
\end{itemize}
\end{multicols}

\end{document}